\documentclass{article}

\usepackage[final]{neurips_2021}

\usepackage[utf8]{inputenc} 
\usepackage[T1]{fontenc}    
\usepackage{hyperref}       
\usepackage{url}            
\usepackage{booktabs}       
\usepackage{amsfonts}       
\usepackage{nicefrac}       
\usepackage{microtype}      
\usepackage{xcolor}         
\usepackage{amsmath}
\usepackage{amssymb}
\usepackage{amsthm}

\usepackage{enumitem}
\usepackage{caption}
\usepackage{subcaption}
\usepackage{xfrac}
\usepackage{tikz,stackengine}
\usepackage{multirow}
\usepackage{lipsum}
\usepackage[ruled]{algorithm2e}
\usepackage{bibentry}
\usepackage{array}
\usepackage{verbatim}

\newcommand{\R}{\mathbf{R}}

\usepackage{mathtools}
\usepackage{tikz}
\usetikzlibrary{shapes}
\usetikzlibrary{positioning}
\usetikzlibrary{calc}

\newtheorem{proposition}{Proposition}

\theoremstyle{definition}

\SetKwInput{kwInput}{\textbf{Input}}
\SetKwInput{kwInit}{\textbf{Initialization}}

\setcitestyle{authoryear,open={(},close={)}}
\hypersetup{
    colorlinks,
    citecolor=blue,
    linkcolor=blue,
}

\title{Flow Network based Generative Models for Non-Iterative Diverse Candidate Generation}

\author{
  Emmanuel Bengio$^{1,2}$, Moksh Jain$^{1,5}$, Maksym Korablyov$^1$\\
  \textbf{Doina Precup$^{1,2,4}$, Yoshua Bengio$^{1,3}$}\\
  $^1$Mila, $^2$McGill University, $^3$Université de Montréal, $^4$DeepMind, $^5$Microsoft \\
}

\begin{document}

\maketitle

\begin{abstract}
  This paper is about the problem of learning a stochastic policy for generating an object (like a molecular graph) from a sequence of actions, such that the probability of generating an object is proportional to a given positive reward for that object. Whereas standard return maximization tends to converge to a single return-maximizing sequence, there are cases where we would like to sample a diverse set of high-return solutions. These arise, for example, in black-box function optimization when few rounds are possible, each with large batches of queries, where the batches should be diverse, e.g., in the design of new molecules. One can also see this as a problem of approximately converting an energy function to a generative distribution. While MCMC methods can achieve that, they are expensive and generally only perform local exploration. Instead, training a generative policy amortizes the cost of search during training and yields to fast generation.  Using insights from Temporal Difference learning, we propose GFlowNet, based on a view of the generative process as a flow network, making it possible to handle the tricky case where different trajectories can yield the same final state, e.g., there are many ways to sequentially add atoms to generate some molecular graph. We cast the set of trajectories as a flow and convert the flow consistency equations into a learning objective, akin to the casting of the Bellman equations into Temporal Difference methods. We prove that any global minimum of the proposed objectives yields a policy which samples from the desired distribution, and demonstrate the improved performance and diversity of GFlowNet on a simple domain where there are many modes to the reward function, and on a molecule synthesis task.
\end{abstract}

\section{Introduction}

The maximization of expected return $R$ in reinforcement learning (RL) is generally achieved by putting all the probability mass of the policy $\pi$ on the highest-return sequence of actions. 
In this paper, we study the scenario where our objective is not to generate the single highest-reward sequence of actions but rather to sample a distribution of trajectories whose probability is proportional to a given positive return or reward function. This can be useful in tasks where exploration is important, i.e., we want to sample from the leading modes of the return function. This is equivalent to the problem of turning an energy function into a corresponding generative model, where the object to be generated is obtained via a sequence of actions. By changing the temperature of the energy function (i.e., scaling it multiplicatively) or by taking the power of the return, one can control how selective the generator should be, i.e., only generate from around the highest modes at low temperature or explore more with a higher temperature. 

A motivating application for this setup is iterative black-box optimization where the learner has access to an oracle which can compute a reward for a large batch of candidates at each round, e.g., in drug-discovery applications. Diversity of the generated candidates is particularly important when the oracle is itself uncertain, e.g., it may consist of cellular assays which is a cheap proxy for clinical trials, or it may consist of the result of a docking simulation (estimating how well a candidate small molecule binds to a target protein) which is a proxy for more accurate but more expensive downstream evaluations (like cellular assays or in-vivo assays in mice). 

When calling the oracle is expensive (e.g. it involves a biological experiment), a standard way~\citep{angermueller-iclr2020} to apply machine learning in such exploration settings is to take the data already collected from the oracle (say a set of $(x,y)$ pairs where $x$ is a candidate solution an $y$ is a scalar evaluation of $x$ from the oracle) and train a supervised proxy $f$ (viewed as a simulator)
which predicts $y$ from $x$. The function $f$
or a variant of $f$ which incorporates uncertainty about its value, like in Bayesian optimization~\citep{Srinivas2010GaussianPO,negoescu2011knowledge}, can then be used as a reward function $R$ to train a generative model or a policy that will produce a batch of candidates for the next experimental assays. Searching for $x$ which maximizes $R(x)$ is not sufficient because we would like to sample for the batch of queries a representative set of $x$'s with high values of $R$, i.e., around modes of $R(x)$. Note that alternative ways to obtain diversity exist, e.g., with batch Bayesian optimization~\citep{kirsch2019batchbald}. An advantage of the proposed approach is that the computational cost is linear in the size of the batch (by opposition with methods which compare pairs of candidates, which is at least quadratic). With the possibility of assays of a hundred thousand candidates using synthetic biology, linear scaling would be a great advantage.

{\bf In this paper, we thus focus on the specific machine learning problem of turning a given positive reward or return function into a generative policy which samples with a probability proportional to the return.} In applications like the one mentioned above, we only apply the reward function after having generated a candidate, i.e., the reward is zero except in a terminal state, and the return is the terminal reward. We are in the so-called episodic setting of RL.

The proposed approach views the probability assigned to an action given a state as the flow associated with a network whose nodes are states, and outgoing edges from that node are deterministic transitions driven by an action (not to be confused with normalizing flows;~\citet{rezende2016variational}). The total flow into the network is the sum of the rewards in the terminal states (i.e., a partition function) and can be shown to be the flow at the root node (or start state). 
The proposed algorithm is inspired by Bellman updates and converges when the incoming and outgoing flow into and out of each state match. A policy which chooses an action with probability proportional to the outgoing flow corresponding to that action is proven to achieve the desired result, i.e., the probability of sampling a terminal state is proportional to its reward. In addition, we show that the resulting setup is off-policy; it converges to the above solution even if the training trajectories come from a different policy, so long as it has large enough support on the state space.

The main contributions of this paper are as follows:
\begin{itemize}[topsep=0pt,itemsep=2pt,parsep=2pt,partopsep=2pt,leftmargin=1cm]
    \item We propose GFlowNet, a novel generative method for unnormalized probability distributions based on flow networks and local flow-matching conditions: the flow incoming to a state must match the outgoing flow.
    \item We prove crucial properties of GFlowNet, including the link between the flow-matching conditions (which many training objectives can provide) and the resulting match of the generated policy with the target reward function. We also prove its offline properties and asymptotic convergence (if the training objective can be minimized). We also demonstrate that previous related work~\citep{buesing2019approximate} which sees the generative process like a tree would fail when there are many action sequences which can lead to the same state.
    \item We demonstrate on synthetic data the usefulness of departing from seeking one mode of the return, and instead seeking to model the entire distribution and all its modes.
    \item We successfully apply GFlowNet to a large scale molecule synthesis domain, with comparative experiments against PPO and MCMC methods.
\end{itemize}

All implementations are available at \url{https://github.com/bengioe/gflownet}.

\newcommand{\amap}{\ensuremath{C}}
\newcommand{\rsum}{\ensuremath{\tilde V}} 
\newcommand{\flowf}{\ensuremath{F}} 

\section{Approximating Flow Network generative models with a TD-like objective}

Consider a discrete set $\cal X$ and policy $\pi(a|s)$
to sequentially build $x \in \cal X$
with probability $\pi(x)$ with 
\begin{equation} 
\pi(x) \approx \frac{R(x)}{Z} = \frac{R(x)}{\sum_{x'\in\cal X} R(x')}
\end{equation}
where $R(x)>0$ is a reward for a terminal state $x$. This would be useful to sample novel drug-like molecules when given a reward function $R$ that scores molecules based on their chemical properties. Being able to sample from the high modes of $R(x)$ would provide diversity in the batches of generated molecules sent to assays. This is in contrast with the typical RL objective of maximizing return which we have found to often end up focusing around one or very few good molecules. In our context, $R(x)$ is a proxy for the actual values obtained from assays, which means it can be called often and cheaply. $R(x)$ is retrained or fine-tuned each time we acquire new data from the assays.

What method should one use to generate batches sampled from $\pi(x)\propto R(x)$? Let's first think of the state space under which we would operate.

Let $\cal S$ denote the set of states and $\cal X \subset S$
denote the set of terminal states. 
Let $\cal A$ be a finite set, the alphabet, ${\cal A}(s) \subseteq {\cal A}$ be the set of allowed actions at state $s$, and let ${\cal A}^*(s)$ be the set of all sequences of actions allowed after state $s$. To every action sequence $\vec{a}=(a_1, a_2, a_3, ..., a_h)$ of $a_i \in{\cal A}, h \leq H$ corresponds a single $x$, i.e. the environment is deterministic so we can define a function $\amap$ mapping a sequence of actions $\vec{a}$ to an $x$. If such a sequence is `incomplete' we define its
reward to be 0.
When the correspondence between action sequences and states is {\bf bijective}, a state $s$ is uniquely described by some sequence $\vec{a}$, and we can visualize the generative process as the traversal of a tree from a single root node to a leaf corresponding to the sequence of actions along the way. 

However, when this correspondence is {\bf non-injective}, i.e. when multiple action sequences describe the same $x$, things get trickier. Instead of a tree, we get a directed acyclic graph or DAG (assuming that the sequences must be of finite length, i.e., there are no deterministic cycles), as illustrated in Figure~\ref{fig:flownet}. 
For example, and of interest here, molecules can be seen as graphs, which can be described in multiple orders (canonical representations such as SMILES strings also have this problem: there may be multiple descriptions for the same actual molecule). 
The standard approach to such a sampling problem is to use iterative MCMC methods \citep{xie2021mars, grathwohl2021oops}. Another option is to relax the desire to have $p(x)\propto R(x)$ and to use non-interative (sequential) RL methods \citep{gottipati2020learning}, but these are at high risk of getting stuck in local maxima and of missing modes. Indeed, in our setting, the policy which maximizes the expected return (which is the expected final reward) generates the sequence with the highest return (i.e., a single molecule).


\vspace*{-2mm}
\subsection{Flow Networks}
\vspace*{-1mm}

In this section we propose the Generative Flow Network framework, or GFlowNet, which enables us to learn policies such that $p(x)\propto R(x)$ when sampled. We first discuss why existing methods are inadequate, and then show how we can use the metaphor of flows, sinks and sources, to construct adequate policies. We then show that such policies can be learned via a flow-matching objective.

With existing methods in the bijective case, one can think of the sequential generation of one $x$ as an episode in a tree-structured deterministic MDP, where all leaves $x$ are terminal states (with reward $R(x)$) and the root is initial state $s_0$. Interestingly, in such a case one can express the pseudo-value of a state $\rsum(s)$ as the sum of all the rewards of the descendants of $s$ \citep{buesing2019approximate}.

In the non-injective case, these methods are inadequate. Constructing $\pi(\tau)\approx R(\tau)/Z$, e.g. as per~\citet{buesing2019approximate}, MaxEnt RL~\citep{haarnoja2017reinforcement}, or via an autoregressive method \citep{nash2019autoregressive,shi2021masked} has a particular problem as shown below: if multiple action sequences $\vec{a}$ (i.e. multiple trajectories $\tau$) lead to a final state $x$, then a serious bias can be introduced in the generative probabilities. 
Let us denote $\vec{a}+\vec{b}$ as
the concatenation of the two sequences of actions $\vec{a}$
and $\vec{b}$, and by extension $s + \vec{b}$
the state reached by applying the actions in $\vec{b}$
from state $s$. 

\begin{proposition}
\label{prop:1}
Let $\amap:{\cal A}^* \mapsto \cal S$ associate each allowed action sequence $\vec{a} \in {\cal A}^*$ to a state \mbox{$s=\amap(\vec{a}) \in \cal S$}.
Let $\rsum: {\cal S} \mapsto \R^+$ associate each state $s \in \cal S$ to $\rsum(s)=\sum_{\vec{b} \in {\cal A}^*(s)} R(s+\vec{b})>0$, where ${\cal A}^*(s)$ is the set of allowed continuations from $s$ and $s+\vec{b}$ denotes the resulting state, i.e., $\rsum(s)$ is the sum of the rewards of all the states reachable from $s$.
Consider a policy $\pi$ which starts from the state corresponding to the empty string
$s_0=\amap(\emptyset)$ and chooses from state $s \in \cal S$ an allowable action $a \in {\cal A}(s)$
with probability $\pi(a | s) = \frac{\rsum(s+a)}{\sum_{b \in {\cal A}(s)} \rsum(s+b)}$. Denote $\pi(\vec{a}=(a_1,\ldots,a_N))=\prod_{i=1}^N \pi(a_i|\amap(a_1, \ldots, a_{i-1}))$ and $\pi(s)$ with $s \in \cal S$ the probability of visiting a state $s$ with this policy. The following then obtains:\\
(a) $\pi(s) = \sum_{\vec{a}_i:\amap(\vec{a}_i)=s} \pi(\vec{a}_i)$.\\
(b) If $\amap$ is bijective, then $\pi(s)=\frac{\rsum(s)}{\rsum(s_0)}$ and as a special case for terminal states $x$, $\pi(x)=\frac{R(x)}{\sum_{x \in \cal X} R(x)}$.\\
(c) If $\amap$ is non-injective and there are $n(x)$ distinct action sequences $\vec{a}_i$ s.t. $\amap(\vec{a}_i)=x$, then \mbox{$\pi(x)=\frac{n(x) R(x)}{\sum_{x' \in \cal X} n(x') R(x')}$}.
\end{proposition}
See Appendix \ref{app:proofs} for the proof.
In combinatorial spaces, such as for molecules, where $\amap$ is non-injective (there are many ways to construct a molecule graph), this can become exponentially bad as trajectory lengths increase. It means that larger molecules would be exponentially more likely to be sampled than smaller ones, just because of the many more paths leading to them.
In this scenario, the pseudo-value $\rsum$ is ``misinterpreting'' the MDP's structure as a tree, leading to the wrong $\pi(x)$.

An alternative is to see the MDP as a \textbf{flow network}, that is, leverage the DAG structure of the MDP, and learn a flow $\flowf$, rather than estimating the pseudo-value $\rsum$ as a sum of descendant rewards, as elaborated below.
We define the flow network as a having a single source, the root node (or initial state) $s_0$ with in-flow $Z$, and one sink for each leaf (or terminal state) $x$ with out-flow $R(x)>0$. We write $T(s,a)=s'$ to denote that the state-action pair $(s,a)$ leads to state $s'$. Note that because $\amap$ is not a bijection, i.e., there are many paths (action sequences) leading to some node, a node can have multiple parents, i.e. $|\{(s,a) \;|\;T(s,a) = s'\}| \geq 1$, except for the root, which has no parent. We write $\flowf(s,a)$ for the flow between node $s$ and node $s'=T(s,a)$, $\flowf(s)$ for the total flow going through $s$\footnote{In some sense, $\flowf(s)$ and $\flowf(s,a)$ are close to  $V$ and $Q$, RL's value and action-value functions. These effectively inform an agent taking decisions at each step of an MDP to act in a desired way. With some work, we can also show an equivalence between $\flowf(s,a)$ and the ``real'' $Q^{\hat\pi}$ of some policy $\hat\pi$ in a modified MDP (see \ref{sec:RL-equivalence}).}\hspace{-0.33em}. This construction is illustrated in Fig. \ref{fig:flownet}.

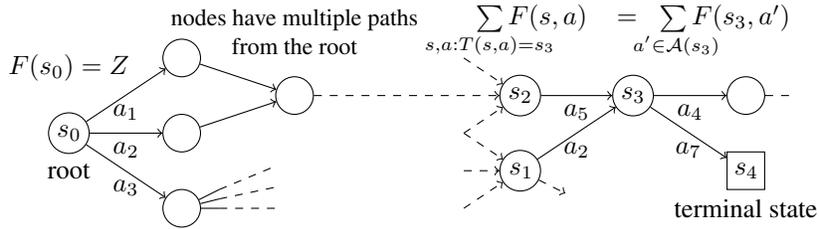
\begin{figure}[h!]
    \centering
    \begin{tikzpicture}
        
\tikzset{state/.style={draw=black, circle, minimum size=0.5cm,
  inner sep=0pt,
  text width=0.5cm, align=center}};
\tikzset{termstate/.style={state, rectangle}};
\node[state,label=below:root] (root) at (0,0) {$s_0$};
\node[above=3mm of root] {$\flowf(s_0)=Z$};
\node[state] (s01) at (1.5,1) {};
\node[state] (s02) at (1.5,0) {};
\node[state] (s03) at (1.5,-1) {};
\draw[->] (root) -> (s01) node [midway, below] {$a_1$};
\draw[->] (root) -> (s02) node [midway, below] {$a_2$};
\draw[->] (root) -> (s03) node [midway, below] {$a_3$};
\node[state] (sterm) at (3,0.5) {}; 
\node[align=center] () at (3, 1.35) {\small nodes have multiple paths\\ \small from the root};
\draw[->] (s01) -> (sterm);
\draw[->] (s02) -> (sterm);
\foreach \i in {0,0.125,0.25}{
  \draw[-] ($(s03)+(0.5,\i)$) -- ++(0.75,\i*1.25) [dashed];
  \draw[-] (s03) -- ++(0.5,\i);
}
\begin{scope}[xshift=2cm]
\node[state] (s1) at (4,-0.5) {$s_1$}; 
\foreach \i in {0,-0.5,0.5}{
  \draw[->, dashed] ($(s1)-(0.75,\i)$) -- (s1);
}

\node[state] (s12) at (4,0.5) {$s_2$}; 
\foreach \i in {-0.5,0.5}{
  \draw[->, dashed] ($(s12)-(0.75,\i)$) -- (s12);
}

\node[align=center] () at (5.5, 1.35) {$\sum\limits_{\mathclap{s,a:T(s,a)=s_3}} \flowf(s,a) \;\;\;=\;\; \sum\limits_{\mathclap{\;a'\in\mathcal{A}(s_3)}}\flowf(s_3,a')$};
\node[state] (s2) at (5.5,0.5) {$s_3$}; 
\draw[->] (s1) -- (s2) node [midway, below] {$a_2$};
\draw[->] (s12) -- (s2) node [midway, below] {$a_5$};
\draw[->, dashed] (s1) -- ++(0.6, -0.3);

\node[state] (s31) at (7,0.5) {}; 
\node[termstate, label=below:terminal state] (s32) at (7,-0.5) {$s_4$}; 
\draw[->] (s2) -- (s31)  node [midway, below] {$a_4$};
\draw[->] (s2) -- (s32)  node [midway, below] {$a_7$};
\draw[-, dashed] (s31.east) -- ++(0.4, 0); 
\end{scope}
\draw[->, dashed] (sterm) -- (s12); 

    \end{tikzpicture}
    \caption{A flow network MDP. Episodes start at source $s_0$ with flow $Z$. Like with SMILES strings, there are no cycles. Terminal states are sinks with out-flow $R(s)$. Exemplar state $s_3$ has parents $\{(s,a)|T(s,a)\!=\!s_3\} \!=\! \{(s_1, a_2), (s_2, a_5)\}$ and allowed actions $\mathcal{A}(s_3) \!=\! \{a_4, a_7\}$. $s_4$ is a terminal sink state with $R(s_4) > 0$ and only one parent. 
    The goal is to estimate $\flowf(s,a)$ such that the flow equations are satisfied for all states: for each node, incoming flow equals outgoing flow.}
    \label{fig:flownet}
\end{figure}

To satisfy flow conditions, we require that for any node, the incoming flow equals the outgoing flow, which is the total flow $\flowf(s)$ of node $s$. Boundary conditions are given by the flow into the terminal nodes $x$, $R(x)$.
Formally, for any node $s'$, we must have that the in-flow
\begin{equation}
    \flowf(s') = \sum_{s,a:T(s,a)=s'} \flowf(s,a)
\end{equation}
equals the out-flow
\begin{equation}
\flowf(s') =  \sum_{a'\in {\cal A}(s')} \flowf(s',a').
\end{equation}
More concisely, with $R(s)=0$ for interior nodes, and  $\mathcal{A}(s)=\varnothing$ for leaf (sink/terminal) nodes, we write the following \textit{flow consistency equations}:
\begin{align}
    \sum_{s,a:T(s,a)=s'} \flowf(s,a) =  R(s') + \sum_{a'\in {\cal A}(s')} \flowf(s',a'). \label{eq:in-out-flow-q-eq}
\end{align}

with $\flowf$ being a flow, $\flowf(s,a)>0 \;\forall s,a$ (for this we needed to constrain $R(x)$ to be positive too). One could include in principle nodes and edges with zero flow but it would make it difficult to talk about the logarithm of the flow, as we do below, 
and such states can always be excluded by the allowed
set of actions for their parent states.
Let us now show that such a flow correctly produces $\pi(x) = R(x)/Z$ when the above flow equations are satisfied. %
\begin{proposition}
\label{prop:2}
Let us define a policy $\pi$ that generates trajectories starting in state $s_0$ by sampling actions $a \in {\cal A}(s)$ according to
\begin{align}
    \pi(a|s) = \frac{\flowf(s,a)}{\flowf(s)}
\label{eq:def-pi}
\end{align}
where $\flowf(s,a)>0$ is the flow through allowed edge $(s,a)$,
$\flowf(s)=R(s) + \sum_{a\in {\cal A}(s)} \flowf(s,a)$ where
$R(s)=0$ for non-terminal nodes $s$ and $\flowf(x)=R(x)>0$ for terminal nodes $x$,
and the flow consistency equation \mbox{$\sum_{s,a:T(s,a)=s'} \flowf(s,a)=R(s') + \sum_{a'\in {\cal A}(s')} \flowf(s',a')$} is satisfied.
Let $\pi(s)$ denote the probability of visiting state $s$
when starting at $s_0$ and following $\pi(\cdot|\cdot)$.
Then \\
(a) \mbox{$\pi(s)=\frac{\flowf(s)}{\flowf(s_0)}$}\\ 
(b) $\flowf(s_0)=\sum_{x \in \cal X} R(x)$\\
(c) \mbox{$\pi(x)=\frac{R(x)}{\sum_{x' \in {\cal X}} R(x')}$}.
\end{proposition}
\begin{proof}
We have $\pi(s_0)=1$ since we always
start in root node $s_0$. Note that $\sum_{x \in \cal X} \pi(x)=1$ because terminal states are mutually exclusive,
but in the case of non-bijective $\amap$, we cannot say
that $\sum_{s \in \cal S} \pi(s)$ equals 1 because
the different states are not mutually exclusive in general.
This notation is different from the one typically
used in RL where $\pi(s)$ refers to the asymptotic
distribution of the Markov chain.
Then 
\begin{equation}
    \pi(s') = \sum_{(a,s):T(s,a)=s'} \pi(a | s) \pi(s)
\end{equation}
i.e., using Eq.~\ref{eq:def-pi},
\begin{equation}
  \pi(s') = \sum_{(a,s):T(s,a)=s'} \frac{\flowf(s,a)}{\flowf(s)} \pi(s).
\end{equation}
We can now conjecture that the statement 
\begin{align}
 \pi(s) = \frac{\flowf(s)}{\flowf(s_0)}   
\end{align}
 is true and prove it by induction. This is trivially true for the root, which is our base statement, since $\pi(s_0)=1$. By induction, we then have that if the statement is true for parents $s$ of $s'$, then
\begin{align}
\pi(s') &= \sum_{s,a:T(s,a)=s'}\frac{\flowf(s,a)}{\flowf(s)}\frac{\flowf(s)}{\flowf(s_0)} = \frac{\sum_{s,a:T(s,a)=s'}\flowf(s,a)}{\flowf(s_0)} = \frac{\flowf(s')}{\flowf(s_0)}
\end{align}

which proves the statement, i.e., the first conclusion (a) of the theorem. We can then apply it to the case of terminal states $x$, whose flow is fixed to $\flowf(x)=R(x)$
and obtain 
\begin{align}
\label{eq:final-pi}
\pi(x)=\frac{R(x)}{\flowf(s_0)}.
\end{align}
Noting that $\sum_{x \in X} \pi(x)=1$ and summing both sides
of Eq.~\ref{eq:final-pi}
over $x$ we thus obtain (b), i.e.,
$\flowf(s_0)=\sum_{x \in \cal X} R(x)$. Plugging this back
into Eq.~\ref{eq:final-pi}, we obtain (c), i.e.,
$\pi(x)=\frac{R(x)}{\sum_{x' \in \cal X} R(x')}$.
\end{proof}
Thus our choice of $\pi(a|s)$ satisfies our desiderata:
it maps a reward function $R$
to a generative model which generates $x$ with
probability $\pi(x) \propto R(x)$, whether $\amap$ is bijective or non-injective
(the former being a special case of the latter, and
we just provided a proof for the general non-injective case).

\vspace*{-2mm}
\subsection{Objective Functions for GFlowNet}
\vspace*{-1mm}
We can now leverage our RL intuitions to create a learning algorithm out of the above theoretical results. In particular, we propose to approximate the flows $\flowf$ such that the flow consistency equations are respected at convergence with enough capacity in our estimator of $\flowf$, just like the Bellman equations for 
temporal-difference (TD) algorithms~\citep{sutton2018reinforcement}.
This could yield the following objective for
a trajectory $\tau$:
\begin{align}
\mathcal{\tilde{L}}_\theta(\tau) = \sum_{s'\in\tau\neq s_0}\left(\smashoperator[r]{\sum_{s,a:T(s,a)=s'}} \flowf_\theta(s,a) -  R(s') - \smashoperator{\sum_{a'\in {\cal A}(s')}} \flowf_\theta(s',a')\right)^{2^{\vphantom{X}}}
.
\label{eq:flow-loss}
\end{align}
One issue from a learning point of view is that the flow will be very large for nodes near the root (early in the trajectory) and tiny for nodes near the leaves (late in the trajectory). 
In high-dimensional spaces where the cardinality of $\cal X$ is exponential (e.g., in the typical number of actions to form an $x$), 
the $\flowf(s,a)$ and $\flowf(s)$ for early states will be exponentially larger than for later states. Since we want $\flowf(s,a)$
to be the output of a neural network, this would lead to serious numerical issues.

To avoid this problem, we define the flow matching objective on a  log-scale, where we match not the incoming and outgoing flows but their logarithms, and we train our predictor to estimate \mbox{$\flowf^{\log}_\theta(s,a)= \log \flowf(s,a)$}, and exponentiate-sum-log the $\flowf^{\log}_\theta$ predictions to compute the loss, yielding the square of a difference of logs:
\begin{equation}
    \mathcal{L}_{\theta,\epsilon}(\tau) = \sum_{\mathclap{s'\in\tau\neq s_0}}\,\left(\log\! \left[\epsilon+\smashoperator[r]{\sum_{\mathclap{s,a:T(s,a)=s'}}} \exp \flowf^{\log}_\theta(s,a)\right] - \log\! \left[ \epsilon + R(s') + \sum_{\mathclap{a'\in {\cal A}(s')}} \exp \flowf^{\log}_\theta(s',a')\right]\right)^2 \label{eq:flow-loss-log}
\end{equation}
which gives equal gradient weighing to large and small magnitude predictions. %
Note that matching the logs of the flows is equivalent to making the ratio of the incoming and outgoing flow closer to 1. To give more weight to errors on large flows and avoid taking the logarithm of a tiny number, we compare $\log(\epsilon+$incoming flow$)$ with $\log(\epsilon+$outgoing flow$)$. It does not change the global minimum, which is still when the flow equations are satisfied, but it avoids numerical issues with taking the log of a tiny flow. The hyper-parameter $\epsilon$ also trades-off how much pressure we put on matching large versus small flows, and in our experiments is set to be close to the smallest value $R$ can take. Since we want to discover the top modes of $R$, it makes sense to care more for the larger flows. Many other objectives are possible for which flow matching is also a global minimum.

An interesting advantage of such objective functions is that they yield off-policy offline methods. The predicted flows $\flowf$ do not depend on the
policy used to sample trajectories (apart from the fact that the samples should sufficiently cover the space of trajectories in order to obtain generalization). This is formalized below, which shows that we can use any broad-support policy to sample training trajectories and still obtain the correct flows
and generative model, i.e., training can be off-policy.
\begin{proposition}
Let trajectories $\tau$ used to train $\flowf_\theta$ 
be sampled from
an exploratory policy $P$ with the same support as 
the optimal $\pi$ defined in Eq.~\ref{eq:def-pi}
for a consistent flow $\flowf^* \in \mathcal{F}^*$. A flow is consistent if Eq.~\ref{eq:in-out-flow-q-eq} is respected. Also assume
that $\exists \theta: \flowf_{\theta} = \flowf^*$, i.e., we choose
a sufficiently rich family of predictors.
Let $\theta^* \in {\rm argmin}_\theta E_{P(\tau)}[ L_\theta(\tau) ]$
a minimizer of the expected training loss.
Let $L_\theta(\tau)$ have the property that
when flows are matched it achieves its lowest
possible value. First, it can be shown
that this property is satisfied for the loss
in Eq.~\ref{eq:flow-loss-log}.
Then
\vspace*{-1mm}
\begin{align}
    \flowf_{\theta^*} &= \flowf^*, \;\;\;{\rm and}\;\;\;
    L_{\theta^*}(\tau) =0 \;\;\; \forall \tau \sim P(\theta),
\end{align}
\vspace*{-6mm}

i.e., a global optimum of the expected loss provides
the correct flows. 
If \mbox{$\pi_{\theta^*}(a|s) = \frac{\flowf_{\theta^*}(s,a)}{\sum_{a' \in {\cal A}(s)} \flowf_{\theta^*}(s,a')}$}
then we also have
\vspace*{-2mm}
\begin{align}
 \pi_{\theta^*}(x)=\frac{R(x)}{Z}.
\end{align}
\end{proposition}
\vspace*{-2mm}
The proof is in Appendix \ref{app:proofs}.
Note that, in RL terms, this method is akin to asynchronous dynamic programming \citep[\S 4.5]{sutton2018reinforcement}, which is an off-policy off-line method which converges provided every state is visited infinitely many times asymptotically. 

\vspace*{-2mm}
\section{Related Work}
\vspace*{-1.5mm}

The objective of training a policy generating states with a probability proportional to rewards was presented by \citet{buesing2019approximate} but the proposed method only makes sense when there is a bijection between action sequences and states. In contrast, GFlowNet is applicable in the more general setting where many paths can lead to the same state. The objective to sample with probability proportional to a given unnormalized positive function is achieved by many MCMC methods~\citep{grathwohl2021oops, dai2020learning}. However, when mixing between modes is challenging (e.g., in high-dimensional spaces with well-separated modes occupying a fraction of the total volume) convergence to the target distribution can be extremely slow. In contrast, GFlowNet is not iterative and amortizes the challenge of sampling from such modes through a training procedure which must be sufficiently exploratory. 

This sampling problem comes up in
molecule generation and has been studied in this context
with numerous generative models~\citep{shi2020graphaf, Jin_2020,luo2021graphdf}, MCMC methods~\citep{seff2019discrete,xie2021mars}, RL~\citep{segler2017generating,decao2018molgan,popova2019molecularrnn,gottipati2020learning,angermueller-iclr2020} and evolutionary methods~\citep{brown2004graph,jensen2019graph,swersky2020amortized}.
Some of these methods rely on a given set of
"positive examples" (high-reward) to train a generative model, thus not taking advantage of the "negative examples" and the continuous nature of the measurements (some examples should be generated more often than others). Others rely on the traditional return maximization objectives of RL, which tends to focus
on one or a few dominant modes, as we find in our
experiments. Beyond molecules, there are previous works generating data non-greedily through RL~\citep{bachman2015data} or energy-based GANs~\citep{dai2017calibrating}.

The objective that we formulate in \eqref{eq:flow-loss-log} may remind the reader of the objective of control-as-inference's Soft Q-Learning~\citep{haarnoja2017reinforcement}, with the difference that we include \emph{all} the parents of a state in the in-flow, whereas Soft Q-Learning only uses the parent contained in the trajectory. Soft Q-Learning induces a different policy, as shown by Proposition \ref{prop:1}, one where $P(\tau)\propto R(\tau)$ rather than $P(x) \propto R(x)$. More generally, we only consider deterministic generative settings whereas RL is a more general framework for stochastic environments.

Literature at the intersection of network flow and deep learning is sparse, and is mostly concerned with solving maximum flow problems \citep{nazemi2012capable,chen2020learning} or classification within existing flow networks \citep{rahul2017deep, pektacs2019deep}. Finally, the idea of accounting for the search space being a DAG rather than a tree in MCTS, known as transpositions~\citep{childs2008transpositions}, also has some links with the proposed method.


\vspace*{-3mm}
\section{Empirical Results}
\vspace*{-1.5mm}

We first verify that GFlowNet works as advertised on an artificial domain small enough to compute the partition function exactly, and compare its abilities to recover modes compared to standard MCMC and RL methods, with its sampling distribution better matching the normalized reward. We find that GFlowNet (A) converges to $\pi(x)\propto R(x)$, (B) requires less samples to achieve some level of performance than MCMC and PPO methods and (C) recovers all the modes and does so faster than MCMC and PPO, both in terms of wall-time and number of states visited and queried.
We then test GFlowNet on a large scale domain, which consists in generating small drug molecule graphs, with a reward that estimates their binding affinity to a target protein (see Appendix \ref{app:chemistry}). We find that GFlowNet finds higher reward and more diverse molecules faster than baselines.

\vspace*{-1.5mm}
\subsection{A (hyper-)grid domain}
\label{exp:single_grid}
\vspace*{-1mm}
Consider an MDP where states are the cells of a $n$-dimensional hypercubic grid of side length $H$. The agent starts at coordinate $x=(0,0,...)$ and is only allowed to increase coordinate $i$ with action $a_i$ (up to $H$, upon which the episode terminates). A \emph{stop} action indicates to terminate the trajectory. There are many action sequences that lead to the same coordinate, making this MDP a DAG.
The reward for ending the trajectory in $x$ is some $R(x)>0$. For MCMC methods, in order to have an ergodic chain, we allow the iteration to decrease coordinates as well, and there is no \emph{stop} action.

We ran experiments with this reward function: 
\begin{equation*}
R(x) = R_0 + 
R_1  \textstyle \prod_i \mathbb{I}(0.25<|x_i/H-0.5|) + 
R_2  \textstyle \prod_i  \mathbb{I}(0.3<|x_i/H-0.5|<0.4) \notag
\stackon[14pt]{\hspace{2cm} \;}
{\begin{tikzpicture}[overlay]
\node (0,0) {};
\node (0,1) {\includegraphics[width=2.3cm]{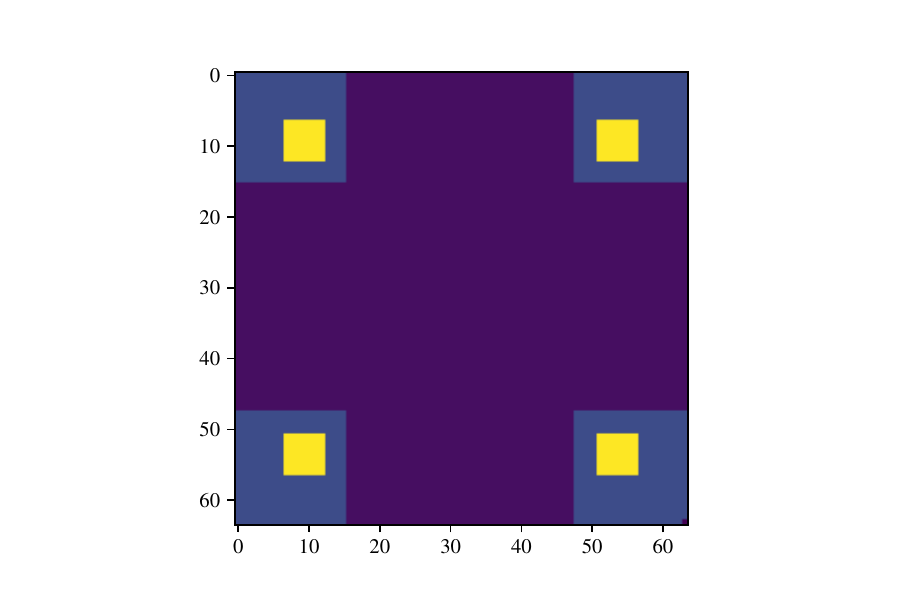}};
\end{tikzpicture}}
\end{equation*}
with $0<R_0\ll R_1<R_2$, pictured when $n=2$ on the right. For this choice of $R$, there are only interesting rewards near the corners of the grid, and there are exactly $2^n$ modes. We set $R_1=1/2,\;R_2=2$. By varying $R_0$ and setting it closer to 0, we make this problem artificially harder, creating a region of the state space which it is undesirable to explore. To measure the performance of a method, we measure the empirical L1 error $\mathbb{E}[|p(x) - \pi(x)|]$. $p(x)=\sfrac{R(x)}{Z}$ is known in this domain, and $\pi$ is estimated by repeated sampling and counting frequencies for each possible $x$. We also measure the number of modes with at least 1 visit as a function of the number of states visited.

We run the above experiment for $R_0\in\{10^{-1}, 10^{-2}, 10^{-3}\}$ with $n=4$, $H=8$. In Fig. \ref{fig:corners_R0} we see that GFlowNet is robust to $R_0$ and obtains a low L1 error, while a Metropolis-Hastings-MCMC based method requires exponentially more samples than GFlowNet to achieve some level of L1 error. This is apparent in Fig. \ref{fig:corners_R0} (with a log-scale horizontal axis) by comparing the slope of progress of GFlowNet (beyond the initial stage) and that of the MCMC sampler. 
We also see that MCMC takes much longer to visit each mode \emph{once} as $R_0$ decreases, while GFlowNet is only slightly affected, with GFlowNet converging to some level of L1 error faster, as per hypothesis (B). This suggests that GFlowNet is robust to the separation between modes (represented by $R_0$ being smaller) and thus recovers all the modes much faster than MCMC (again, noting the log-scale of the horizontal axis).

To compare to RL, we run PPO~\citep{schulman2017proximal}. To discover all the modes in a reasonable time, we need to set the entropy maximization term much higher ($0.5$) than usual ($\ll 1$). 
We verify that PPO is not overly regularized by comparing it to a random agent. PPO finds all the modes faster than uniform sampling, but much more slowly than GFlowNet, and is also robust to the choice of $R_0$. This and the previous result validates hypothesis (C). We also run SAC~\citep{haarnoja2018soft}, finding similar or worse results. We provide additional results and discussion in Appendix \ref{app:toy}.

\begin{figure}[h]
    \centering
    \vspace*{-2mm}
    \makebox[0.9\textwidth][c]{\includegraphics[width=0.95\linewidth]{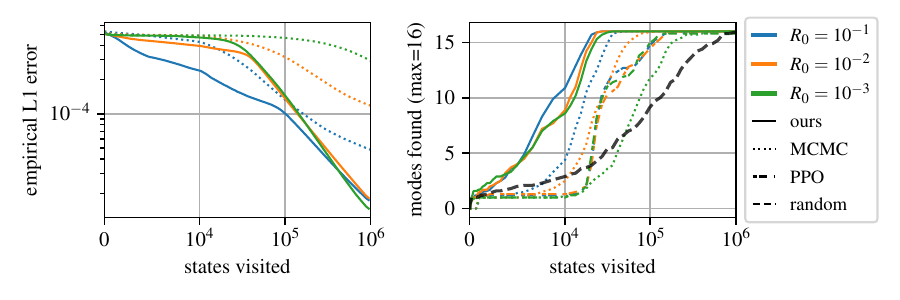}}
    \vspace*{-3mm}
    \caption{Hypergrid domain. Changing the task difficulty $R_0$ to illustrate the advantage of 
    GFlowNet over others. We see that as $R_0$ gets smaller, MCMC struggles to fit the distribution because it struggles to visit all the modes. PPO also struggles to find all the modes, and requires very large entropy regularization, but is robust to the choice of $R_0$. We plot means over 10 runs for each setting.
    }
    \vspace*{-3mm}
    \label{fig:corners_R0}
\end{figure}

\vspace*{-1.5mm}
\subsection{Generating small molecules}
\vspace*{-1mm}
Here our goal is to generate a diverse set of small molecules that have a high reward. We define a large-scale environment which allows an agent to sequentially generate molecules. This environment is challenging, with up to $10^{16}$ states and between $100$ and $2000$ actions depending on the state.

We follow the framework of \citet{Jin_2020} and generate molecules by parts using a predefined vocabulary of building blocks that can be joined together forming a \emph{junction tree} (detailed in \ref{app:chemistry}). This is also known as fragment-based drug design \citep{kumar2012fragment,xie2021mars}. Generating such a graph can be described as a sequence of additive edits: given a molecule and constraints of chemical validity, we choose an atom to attach a block to. The action space is thus the product of choosing where to attach a block and choosing which block to attach. There is an extra action to stop the editing sequence.
This sequence of edits yields a DAG MDP, as there are multiple action sequences that lead to the same molecule graph, 
and no edge removal actions, which prevents cycles.

The reward is computed with a pretrained \emph{proxy} model  that predicts the binding energy of a molecule to a particular protein target (soluble epoxide hydrolase, sEH, see \ref{app:chemistry}). Although computing binding energy is computationally expensive, we can call this proxy cheaply. Note that for realistic drug design, we would need to consider many more quantities such as drug-likeness \citep{bickerton2012quantifying}, toxicity, or synthesizability. Our goal here is not solve this problem, and our work situates itself within such a larger project. Instead, we want to show that given a proxy $R$ in the space of molecules, we can quickly match its induced distribution $\pi(x)\propto R(x)$ and find many of its modes.

We parameterize the proxy with an MPNN \citep{gilmer2017neural} over the atom graph. 
Our flow predictor $\flowf_\theta$ is parameterized similarly to MARS \citep{xie2021mars}, with an MPNN, but over the junction tree graph (the graph of blocks), which had better performance. 
For fairness, this architecture is used for both GFlowNet and the baselines. Complete details can be found in Appendix \ref{app:impl-details}.

We pretrain the proxy with a semi-curated semi-random dataset of 300k molecules (see \ref{app:impl-details}) down to a test MSE of 0.6; molecules are scored according to the docking score~\citep{trott2010autodock}, renormalized so that most scores fall between 0 and 10 (to have $R(x)>0$). We plot the dataset's reward distribution in Fig. \ref{fig:histo-beta}.
We train all generative models with up to $10^6$ molecules. During training, sampling follows exploratory policy $P(a|s)$ which is a mixture between $\pi(a|s)$ (Eq. \ref{eq:def-pi}), used with probability 0.95, and a uniform distribution over allowed actions with probability 0.05. 

\begin{figure}[h!]
\vspace*{-1mm}
    \centering
    \begin{minipage}[t]{.48\textwidth}
      \centering
      \includegraphics[width=0.99\linewidth]{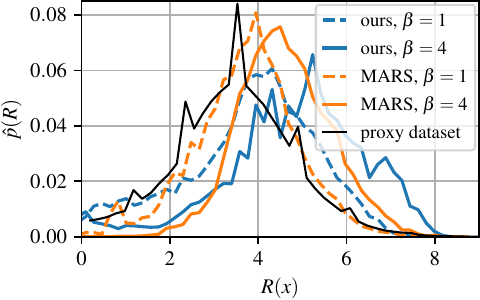}
      \captionof{figure}{Empirical density of rewards. We verify that GFlowNet is consistent by training it with $R^\beta$, $\beta=4$, which has the hypothesized effect of shifting the density to the right.}
      \label{fig:histo-beta}
    \end{minipage}%
    \hspace{0.03\textwidth}
    \begin{minipage}[t]{.48\textwidth}
      \centering
      \includegraphics[width=0.99\linewidth]{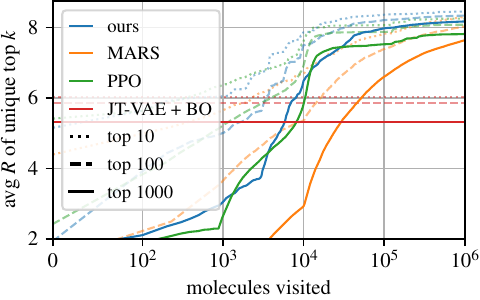}
      \captionof{figure}{The average reward of the top-$k$ as a function of learning (averaged over 3 runs). Only unique hits are counted. Note the log scale. Our method finds more unique good molecules faster.}
      \label{fig:top-k}
    \end{minipage}%
    \vspace*{-3mm}
\end{figure}

\textbf{Experimental results} $\,$ 
In Fig. \ref{fig:histo-beta} we show the empirical distribution of rewards in two settings; first when we train our model with $R(x)$, then with $R(x)^\beta$. If GFlowNet learns a reasonable policy $\pi$, this should shift the distribution to the right. This is indeed what we observe. We compare GFlowNet to MARS~\citep{xie2021mars}, known to work well in the molecule domain, and observe the same shift. Note that GFlowNet finds more high reward molecules than MARS with these $\beta$ values; this is consistent with the hypothesis that it finds high-reward modes faster (since MARS is an MCMC method, it would eventually converge to the same distribution, but takes more time).

In Fig. \ref{fig:top-k}, we show the average reward of the top-$k$ molecules found so far, without allowing for duplicates (
based on SMILES). We compare GFlowNet with MARS, PPO, and JT-VAE~\citep{Jin_2020} with Bayesian Optimization. As expected, PPO plateaus after a while; RL tends to be satisfied with good enough trajectories unless it is strongly regularized with exploration mechanisms. For GFlowNet and for MARS, the more molecules are visited, the better they become, with a slow convergence towards the proxy's max reward. Given the same compute time, JT-VAE+BO generates only about $10^3$ molecules (due to its expensive Gaussian Process) and so does not perform well.

The maximum reward in the proxy's dataset is 10, with only 233 examples above 8. In our best run, we find 2339 unique molecules during training with a score above 8, only 39 of which are in the dataset. We compute the average pairwise Tanimoto similarity for the top 1000 samples: GFlowNet has a mean of $0.44\pm 0.01$, PPO, $0.62\pm0.03$, and MARS, $0.59\pm 0.02$ (mean and std over 3 runs). As expected, our MCMC baseline (MARS) and RL baseline (PPO) find less diverse candidates. 
We also find that GFlowNet discovers \textbf{many} more modes ($>\!1500$ with $R\!>\!8$ vs $<\!100$ for MARS). This is shown in Fig. \ref{fig:mol-modes-scaff} where we consider a mode to be a Bemis-Murcko scaffold~\citep{bemis1996properties}, counted for molecules above a certain reward threshold. 
We provide additional insights into how GFlowNet matches the rewards in Appendix \ref{app:flownet-more-results}.

\begin{figure}[h]
\vspace{-2mm}
    \centering
    \includegraphics[width=\linewidth]{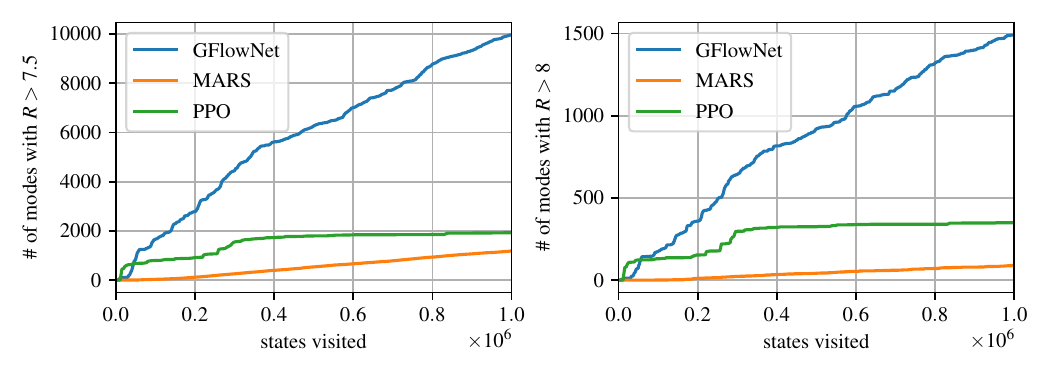}
    \caption{Number of diverse Bemis-Murcko scaffolds found above reward threshold $T$ as a function of the number of molecules seen. Left, $T=7.5$. Right, $T=8$. 
    }
    \label{fig:mol-modes-scaff}
\vspace{-2mm}
\end{figure}

\vspace*{-2mm}
\subsection{Multi-Round Experiments}
\vspace*{-1.5mm}
To demonstrate the importance of diverse candidate generation in an active learning setting, we consider a sequential acquisition task. We simulate the setting where there is a limited budget for calls to the true oracle $O$. We use a proxy $M$ initialized by training on a limited dataset of $(x, R(x))$ pairs $D_0$, where $R(x)$ is the true reward from the oracle. The generative model ($\pi_{\theta}$) is trained to fit to the unnormalized probability function learned by the proxy $M$. We then sample a batch $B=\{x_1, x_2, \dots x_k\}$ where $x_i\sim \pi_{\theta}$, which is evaluated with the oracle $O$. The proxy $M$ is updated with this newly acquired and labeled batch, and the process is repeated for $N$ iterations. We discuss the experimental setting in more detail in  Appendix~\ref{app:multi-round}. 

\begin{figure}[h!]
    \centering
    \vspace*{-3mm}
    \begin{minipage}[t]{.48\textwidth}
    \centering
      \includegraphics[width=\textwidth]{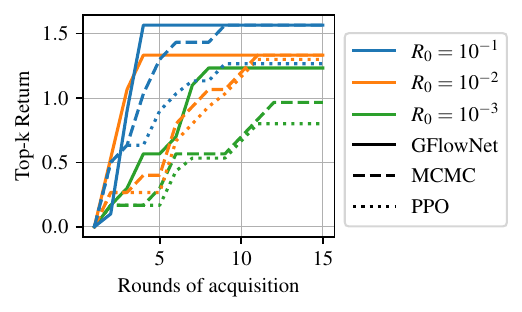}
      \vspace*{-4mm}
      \captionof{figure}{The top-k return (mean over 3 runs) in the 4-D Hyper-grid task with active learning. GFlowNet gets the highest return faster.}
      \label{fig:al_toy}
    \end{minipage}%
    \hspace{0.03\textwidth}
    \begin{minipage}[t]{.48\textwidth}
      \centering
      \includegraphics[width=\textwidth]{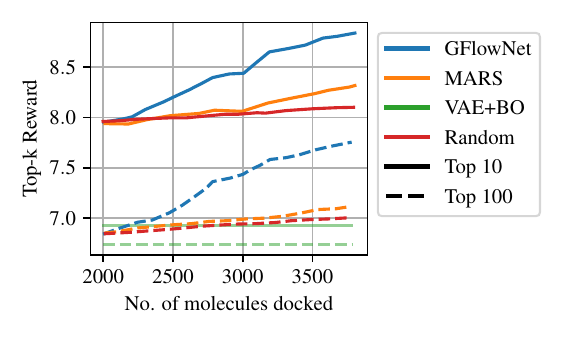}
      \vspace*{-4mm}
      \captionof{figure}{The top-k docking reward (mean over 3 runs) in the molecule task with active learning. GFlowNet consistently generates better samples.}
      \label{fig:top-k-active}
    \end{minipage}%
    \vspace*{-1mm}
\end{figure}
\mathchardef\mhyphen="2D
\textbf{Hyper-grid domain} $\,$
We present results for the multi-round task in the 4-D hyper-grid domain in Figure~\ref{fig:al_toy}. We use a Gaussian Process~\citep{Williams1995GaussianPF} as the proxy. We compare the {\em Top-k Return} for all the methods, which is defined as $\mathrm{mean}(\mathrm{top}\,\mhyphen k(D_i)) - \mathrm{mean}(\mathrm{top}\,\mhyphen k(D_{i-1}))$, where $D_i$ is the dataset of points acquired until step $i$, and $k=10$ for this experiment. The initial dataset $D_0$ ($|D_0|=512$) is the same for all the methods compared. We observe that GFlowNet consistently outperforms the baselines in terms of return over the initial set. We also observe that the mean pairwise L2-distance between the $\mathrm{top}\,\mhyphen k$ points at the end of the final round is $0.83 \pm 0.03$, $0.61 \pm 0.01$ and $0.51 \pm 0.02$ for GFlowNet, MCMC and PPO respectively. This demonstrates the ability of GFlowNet to capture the modes, even in the absence of the true oracle, as well as the importance of capturing this diversity in multi-round settings.

\textbf{Small Molecules} $\,$
For the molecule discovery task, we initialize an MPNN proxy to predict docking scores from AutoDock \citep{trott2010autodock}, with $|D_0| = 2000$ molecules. At the end of each round we generate $200$ molecules which are evaluated with AutoDock and used to update the proxy. Figure~\ref{fig:top-k-active} shows GFlowNet discovers molecules with significantly higher energies than the initial set $D_0$. It also consistently outperforms MARS as well as Random Acquisition. PPO training was unstable and diverged consistently so the numbers are not reported. The mean pairwise Tanimoto similarity in the initial set is 0.60. At the end of the final round, it is $0.54 \pm 0.04$ for GFlowNet and $0.64 \pm 0.03$ for MARS. This further demonstrates the ability of GFlowNet to generate diverse candidates, which ultimately helps improve the final performance on the task. Similar to the single step setting, we observe that JT-VAE+BO is only able to generate $10^3$ molecules with similar compute time, and thus performs poorly.

\vspace*{-2mm}
\section{Discussion \& Limitations}
\vspace{-1mm}
In this paper we have introduced a novel TD-like objective for learning a flow for each state and \textit{(state, action)} pair such that policies sampling actions proportional to these flows draw terminal states in proportion to their reward. This can be seen as an alternative approach to turn an energy function into a fast generative model, without the need for an iterative method like that needed with MCMC methods, and with the advantage that when training succeeds, the policy generates a great diversity of samples near the main modes of the target distribution without being slowed by issues of mixing between modes.

{\bf Limitations.} One downside of the proposed method is that, as for TD-based methods, the use of bootstrapping may cause optimization challenges~\citep{kumar2020implicit,bengio2020interference} and limit its performance.
In applications like drug discovery, sampling from the regions surrounding each mode is already an important advantage, but future work should investigate how to combine such a generative approach to local optimization in order to refine the generated samples and approach the local maxima of reward while keeping the batches of candidates diverse.

{\bf Negative Social Impact. }The authors do not foresee negative social impacts of this work specifically.  




\begin{ack}
This research was enabled in part by computational resources provided by Calcul Québec (\url{www.calculquebec.ca}) and Compute Canada (\url{www.computecanada.ca}). 
All authors are funded by their primary academic institution. We also acknowledge funding from Samsung Electronics Co., Ldt., CIFAR and IBM.

The authors are grateful to Andrei Nica for generating the molecule dataset, to Maria Kadukova for advice on molecular docking, to Harsh Satija for feedback on the paper, as well as to all the members of the Mila Molecule Discovery team for the many research discussions on the challenges we faced.

\end{ack}

\section*{Author Contributions}

EB and YB contributed to the original idea, and wrote most sections of the paper. YB wrote the proofs of Propositions 1-3, EB the proof of Proposition 4. EB wrote the code and ran experiments for sections 4.1 (hypergrid) and 4.2 (small molecules). MJ wrote the code and ran experiments for section 4.3 (multi-round) and wrote the corresponding results section of the paper. MK wrote the biochemical framework upon which the molecule experiments are built, assisted in debugging and running experiments for section 4.3, implemented mode-counting routines used in 4.2, and wrote the biochemical details of the paper.

MK, DP and YB provided supervision for the project. All authors contributed to proofreading and editing the paper.

\bibliography{main}
\bibliographystyle{plainnat}

\newpage
\appendix

\section{Appendix}

All our ML code uses the PyTorch~\citep{pytorch2019} library. We reimplement RL and other baselines. We use the AutoDock Vina~\citep{trott2010autodock} library for binding energy estimation and RDKit~\citep{rdkit} for chemistry routines.

Running all the molecule experiments presented in this paper takes an estimated 26 GPU days. We use a cluster with NVidia V100 GPUs. The grid experiments take an estimated 8 CPU days (for a single-core).

All implementations are available at \url{https://github.com/bengioe/gflownet}.

\subsection{Proofs}
\label{app:proofs}

\setcounter{proposition}{0}
\begin{proposition}
Let $\amap:{\cal A}^* \mapsto \cal S$ associate each allowed action sequence $\vec{a} \in {\cal A}^*$ to a state \mbox{$s=\amap(\vec{a}) \in \cal S$}.
Let $\rsum: {\cal S} \mapsto \R^+$ associate each state $s \in \cal S$ to $\rsum(s)=\sum_{\vec{b} \in {\cal A}^*(s)} R(s+\vec{b})>0$, where ${\cal A}^*(s)$ is the set of allowed continuations from $s$ and $s+\vec{b}$ denotes the resulting state, i.e., $\rsum(s)$ is the sum of the rewards of all the states reachable from $s$.
Consider a policy $\pi$ which starts from the state corresponding to the empty string
$s_0=\amap(\emptyset)$ and chooses from state $s \in \cal S$ an allowable action $a \in {\cal A}(s)$
with probability $\pi(a | s) = \frac{\rsum(s+a)}{\sum_{b \in {\cal A}(s)} \rsum(s+b)}$. Denote $\pi(\vec{a}=(a_1,\ldots,a_N))=\prod_{i=1}^N \pi(a_i|\amap(a_1, \ldots, a_{i-1}))$ and $\pi(s)$ with $s \in \cal S$ the probability of visiting a state $s$ with this policy. The following then obtains:\\
(a) $\pi(s) = \sum_{\vec{a}_i:\amap(\vec{a}_i)=s} \pi(\vec{a}_i)$.\\
(b) If $\amap$ is bijective, then $\pi(s)=\frac{\rsum(s)}{\rsum(s_0)}$ and as a special case for terminal states $x$, $\pi(x)=\frac{R(x)}{\sum_{x \in \cal X} R(x)}$.\\
(c) If $\amap$ is non-injective and there are $n(x)$ distinct action sequences $\vec{a}_i$ s.t. $\amap(\vec{a}_i)=x$, then \mbox{$\pi(x)=\frac{n(x) R(x)}{\sum_{x' \in \cal X} n(x') R(x')}$}.
\end{proposition}
\begin{proof}
Since $s$ can be reached (from $s_0$) according to any of the
action sequences $\vec{a}_i$ such that $\amap(\vec{a}_i)=s$ and they are mutually exclusive and cover all the possible ways of reaching $s$, the probability that $\pi$ visits state $s$ is simply $\sum_{\vec{a}_i:\amap(\vec{a}_i)=s} \pi(\vec{a}_i)$, i.e., we obtain (a).\\
If $\amap$ is bijective, it means that there is only one such action sequence $\vec{a}=(a_1,\ldots,a_N)$ landing in state $s$, and the set of action sequences and states forms a tree rooted at $s_0$.
hence by (a) we get that $\pi(s) = \pi(\vec{a})$.
First note that because $\rsum(s)=\sum_{\vec{b} \in {\cal A}^*(s)} R(s+\vec{b})$, i.e., $\rsum(s)$ is the sum of the terminal rewards
for all the leaves rooted at $s$, we have that $\rsum(s)=\sum_{b \in {\cal A}(s)} V(s+b)$.
Let us now prove by induction that $\pi(s) = \frac{\rsum(s)}{\rsum(s_0)}$.
It is true for $s=s_0$ since $\pi(s_0=1)$ (i.e., every trajectory includes $s_0$). Assuming it is true for $s'=\amap(a_1,\ldots,a_{N-1})$,
consider $s=\amap(a_1,\ldots,a_N)$:
$$
\pi(s)=\pi(a_N|s')\pi(s')=\frac{\rsum(s)}{\sum_{b \in {\cal A}(s')} \rsum(s'+b)} \frac{\rsum(s')}{\rsum(s_0)}.
$$
Using our above result that $\rsum(s)=\sum_{b \in {\cal A}(s)} \rsum(s+b)$, we thus obtain a cancellation of $\rsum(s')$ with $\sum_{b \in {\cal A}(s')} \rsum(s'+b)$ and obtain
\begin{align}
 \pi(s) = \frac{\rsum(s)}{\rsum(s_0)},
 \label{eq:pi-property}
\end{align}
 proving that the recursion holds. We already know from the definition of $\rsum$ that $\rsum(s_0)=\sum_{x \in \cal X} R(x)$, so for the special case of $x$ a terminal state, $\rsum(x)=R(x)$ and Eq.~\ref{eq:pi-property}
 becomes $\pi(x)=\frac{R(x)}{\sum_{x' \in \cal X} R(x')}$, which finishes to prove (b).\\
 On the other hand, if $\amap$ is non-injective, the set of paths forms a DAG, and not a tree. Let us transform the DAG into a tree by creating a new state-space (for the tree version) which is the action sequence itself. Note how the same original leaf node $x$ is now repeated $n(x)$ times in the tree (with leaves denoted by action sequences $\vec{a}_i$) if there are $n(x)$ action sequences leading to $x$ in the DAG. With the same definition of $\rsum$ and $\pi(a|s)$
 but in the tree, we obtain all the results from (b) (which are applicable because we have a tree),
 and in particular $\pi(\vec{a}_i)$ under the tree is proportional to $R(x')=R(x)$. Applying (a), we see that $\pi(x) \propto n(x) R(x)$, which proves (c).
\end{proof}

\setcounter{proposition}{2}
\begin{proposition}
Let trajectories $\tau$ used to train $\flowf_\theta$ 
be sampled from
an exploratory policy $P$ with the same support as 
the optimal $\pi$ defined in Eq.~\ref{eq:def-pi}
for a consistent flow $\flowf^* \in \mathcal{F}^*$. A flow is consistent if Eq.~\ref{eq:in-out-flow-q-eq} is respected. Also assume
that $\exists \theta: \flowf_{\theta} = \flowf^*$, i.e., we choose
a sufficiently rich family of predictors.
Let $\theta^* \in {\rm argmin}_\theta E_{P(\tau)}[ L_\theta(\tau) ]$
a minimizer of the expected training loss.
Let $L_\theta(\tau)$ have the property that
when flows are matched it achieves its lowest
possible value. First, it can be shown
that this property is satisfied for the loss
in Eq.~\ref{eq:flow-loss-log}.
Then
\begin{align}
    \flowf_{\theta^*} &= \flowf^*, \;\;\;{\rm and}\\
    L_{\theta^*}(\tau) &=0 \;\;\; \forall \tau \sim P(\theta),
\end{align}

i.e., a global optimum of the expected loss provides
the correct flows. 
If 
\begin{equation}
    \pi_{\theta^*}(a|s) = \frac{\flowf_{\theta^*}(s,a)}{\sum_{a' \in {\cal A}(s)} \flowf_{\theta^*}(s,a')}
\end{equation}
then we also have
\vspace*{-2mm}
\begin{align}
 \pi_{\theta^*}(x)=\frac{R(x)}{Z}.
\end{align}
\end{proposition}

\begin{proof}
 A per-trajectory loss of 0 can be achieved by choosing a $\theta$ such that $\flowf_{\theta}=\flowf^*$ (which we assumed was possible), since this makes the incoming flow equal the outgoing flow. Note that there always exists a solution $\flowf^*$ in the space of allow possible flow functions which satisfies the flow equations (incoming = outgoing) by construction of flow networks with only a constraint on the flow in the terminal nodes (leaves).
 Since having $L_\theta(\tau)$ equal to 0 for all $\tau \sim P(\theta)$ makes the expected loss 0, and this is the lowest achievable value (since $L_\theta(\tau) \geq 0 \;\;\forall \theta$), it means that such a $\theta$ is a global 
minimizer of the expected loss, and we can denote it $\theta^*$.
If we optimize $\flowf$ in function space, we can directly set to 0 the gradient of the loss with respect to $\flowf(s,a)$ separately, and find a solution. 

Since we have chosen $P$ with support large enough to include 
all the trajectories leading to a terminal state $R(x)>0$,
it means that $L_\theta(\tau)=0$ for all these trajectories
and that $\flowf_\theta=\flowf^*$ for all nodes on these trajectories.
We can then apply Proposition~\ref{prop:2} (since the flows
match everywhere and we have defined the policy correspondingly,
as per Eq.~\ref{eq:def-pi}). We then obtain the conclusion by
applying result (c) from Proposition~\ref{prop:2}.
\end{proof}

Note that in the general case, an infinite number of solutions exist. Consider the case where two trajectories are possible, say $s_0,a_1,s_A,a_2,s_T$ and $s_0,a_3,s_B,a_4,s_T$, and both lead to the same terminal state $s_T$ with reward $r$. Then a valid solution solves the constrained system of equations $F(s_A) + F(s_B) = r, F(s_A)>0, F(s_B)>0$, and we see that there is an infinite number of solutions described by one parameter $u$ where $F(s_A) = u, F(s_B)=r-u\; u\in[0,r]$.

\subsection{Action-value function equivalence}
\label{sec:RL-equivalence}

Here we show that the flow $\flowf(s,a)$ that the proposed method learns can correspond to a ``real'' action-value function $\hat{Q}^\mu(s,a)$ in an RL sense, for a policy $\mu$.

First note that this is in a way trivially true: in inverse RL \citep{ng2000algorithms} there typically exists an infinite number of solutions to defining $\hat R$ from a policy $\pi$ such that $\pi = \arg\max_{\pi_i} V^{\pi_i}(s; \hat R)\;\forall s$, where $V^{\pi_i}(s; \hat R)$ is the value function at $s$ for reward function $\R$.

More interesting is the case where $\flowf(s,a;R)$ obtained from computing the flow corresponding to $R$ is exactly equal to some $Q^\mu(s,a;\hat{R})$ modulo a multiplicative factor $f(s)$ . What are $\mu$ and $\hat R$?

In the bijective case a simple answer exists. 

\begin{proposition}
\label{prop:q-equivalence}
Let $\mu$ be the uniform policy such that $\mu(a|s) = 1/|\mathcal{A}(s)|$, let $f(x)=\prod_{t=0}^{n} |\mathcal{A}(s_t)|$ when $x\equiv(s_0, s_1, ..., s_n)$, and let $\hat R(x) = R(x) f(s_{n-1})$,  then $Q^\mu(s,a;\hat{R})=\flowf(s,a;R)f(s)$.
\end{proposition}
\begin{proof}
 By definition of the action-value function in terms of the action-value at the next step and by definition of $\mu$:
 \begin{equation}
     Q^\mu(s,a;\hat R) = \hat R(s') + \frac{1}{|\mathcal{A}(s')|} \sum_{a'\in\mathcal{A}(s')} Q^\mu(s',a';\hat{R}) \label{eq:q-equiv}
 \end{equation}
where $s'=T(s,a)$, as the environment is deterministic and has a tree structure.
 
For some leaf $s'$, $Q^\mu(s,a;\hat R)=\hat R(s')=R(s') f(s)$. Again for some leaf $s'$, the flow is $\flowf(s,a;R)=R(s')$. Thus $Q^\mu(s,a;\hat R)=\flowf(s,a;R)f(s)$. Suppose \eqref{eq:q-equiv} is true, then by induction for a non-leaf $s'$:
 \begin{align}
     Q^\mu(s,a;\hat R) &= \hat R(s') + \frac{1}{|\mathcal{A}(s')|} \sum_{a'\in\mathcal{A}(s')} Q^\mu(s',a';\hat{R})\\
     Q^\mu(s,a;\hat R) &= 0 + \frac{1}{|\mathcal{A}(s')|} \sum_{a'\in\mathcal{A}(s')} \flowf(s',a';R)f(s')
 \end{align}
we know from Eq \ref{eq:in-out-flow-q-eq} that
\begin{equation}
    \flowf(s,a;R) = \sum_{a'\in\mathcal{A}(s')} \flowf(s',a';R)
\end{equation}
and since  $f(s') = f(s)|\mathcal{A}(s')|$, we have that:
\begin{align}
     Q^\mu(s,a;\hat R) &=  \frac{\flowf(s,a;R)f(s')}{|\mathcal{A}(s')|}\\
     &=  \frac{\flowf(s,a;R)f(s)|\mathcal{A}(s')|}{|\mathcal{A}(s')|}\\
     &= \flowf(s,a;R)f(s)
\end{align}
\end{proof}

Thus we have shown that the flow in a bijective case corresponds to the action-value of the uniform policy. This result suggests that the policy evaluation of the uniform policy learns something non-trivial in the tree MDP case. Perhaps such a quantity could be used in other interesting ways.

In the non-injective case, since an infinite number of valid flows exists, it's not clear that such a simple equivalence always exists.

As a particular case, consider the flow $\flowf$ which assigns exactly 0 flow to edges that would induce multiple paths to any node. In other words, consider the flow which induces a tree, i.e. a bijection between action sequences and states, by disallowing flow between edges not in that bijection. By Proposition \ref{prop:q-equivalence}, we can recover some valid $Q^\mu$.

Since there is at least one flow for which this equivalence exists, we conjecture that more general mappings between flows and action-value functions exist.

\textbf{Conjecture} $\,$ \textit{There exists $f$ a function of $n(s)$ the number of paths to $s$, $\mathcal{A}(s)$, and $n_p(s)=|\{(p,a)|T(p,a)=s\}|$ the number of parents of $s$, such that $f(s, n(s), n_p(s), \mathcal{A}(s)) Q^\mu(s,a;\hat R)=\flowf(s,a;R)$ and $\hat R(x) = R(x) f(x)$ for the uniform policy $\mu$ and for some valid flow $\flowf(s,a;R)$.}

\subsection{Molecule domain details}
\label{app:chemistry}

We allow the agent to choose from a library of 72 predefined blocks. We duplicate blocks from the point of view of the agent to allow attaching to different symmetry groups of a given block. This yields a total of 105 actions per stem; stems are atoms where new blocks can be attached to. We choose the blocks via the process suggested by \citet{Jin_2020} over the ZINC dataset \citep{sterling2015zinc}. We allow the agent to generate up to 8 blocks.

The 72 block SMILES are \texttt{Br}, \texttt{C}, \texttt{C\#N}, \texttt{C1=CCCCC1}, \texttt{C1=CNC=CC1}, \texttt{C1CC1}, \texttt{C1CCCC1}, \texttt{C1CCCCC1}, \texttt{C1CCNC1}, \texttt{C1CCNCC1}, \texttt{C1CCOC1}, \texttt{C1CCOCC1}, \texttt{C1CNCCN1}, \texttt{C1COCCN1}, \texttt{C1COCC[NH2+]1}, \texttt{C=C}, \texttt{C=C(C)C}, \texttt{C=CC}, \texttt{C=N}, \texttt{C=O}, \texttt{CC}, \texttt{CC(C)C}, \texttt{CC(C)O}, \texttt{CC(N)=O}, \texttt{CC=O}, \texttt{CCC}, \texttt{CCO}, \texttt{CN}, \texttt{CNC}, \texttt{CNC(C)=O}, \texttt{CNC=O}, \texttt{CO}, \texttt{CS}, \texttt{C[NH3+]}, \texttt{C[SH2+]}, \texttt{Cl}, \texttt{F}, \texttt{FC(F)F}, \texttt{I}, \texttt{N}, \texttt{N=CN}, \texttt{NC=O}, \texttt{N[SH](=O)=O}, \texttt{O}, \texttt{O=CNO}, \texttt{O=CO}, \texttt{O=C[O-]}, \texttt{O=PO}, \texttt{O=P[O-]}, \texttt{O=S=O}, \texttt{O=[NH+][O-]}, \texttt{O=[PH](O)O}, \texttt{O=[PH]([O-])O}, \texttt{O=[SH](=O)O}, \texttt{O=[SH](=O)[O-]}, \texttt{O=c1[nH]cnc2[nH]cnc12}, \texttt{O=c1[nH]cnc2c1NCCN2}, \texttt{O=c1cc[nH]c(=O)[nH]1}, \texttt{O=c1nc2[nH]c3ccccc3nc-2c(=O)[nH]1}, \texttt{O=c1nccc[nH]1}, \texttt{S}, \texttt{c1cc[nH+]cc1}, \texttt{c1cc[nH]c1}, \texttt{c1ccc2[nH]ccc2c1}, \texttt{c1ccc2ccccc2c1}, \texttt{c1ccccc1}, \texttt{c1ccncc1}, \texttt{c1ccsc1}, \texttt{c1cn[nH]c1}, \texttt{c1cncnc1}, \texttt{c1cscn1}, \texttt{c1ncc2nc[nH]c2n1}.

We illustrate these building blocks and their attachment points in Figure \ref{fig:blocks}.

\begin{figure}[h]
    \centering
    \includegraphics[width=\linewidth]{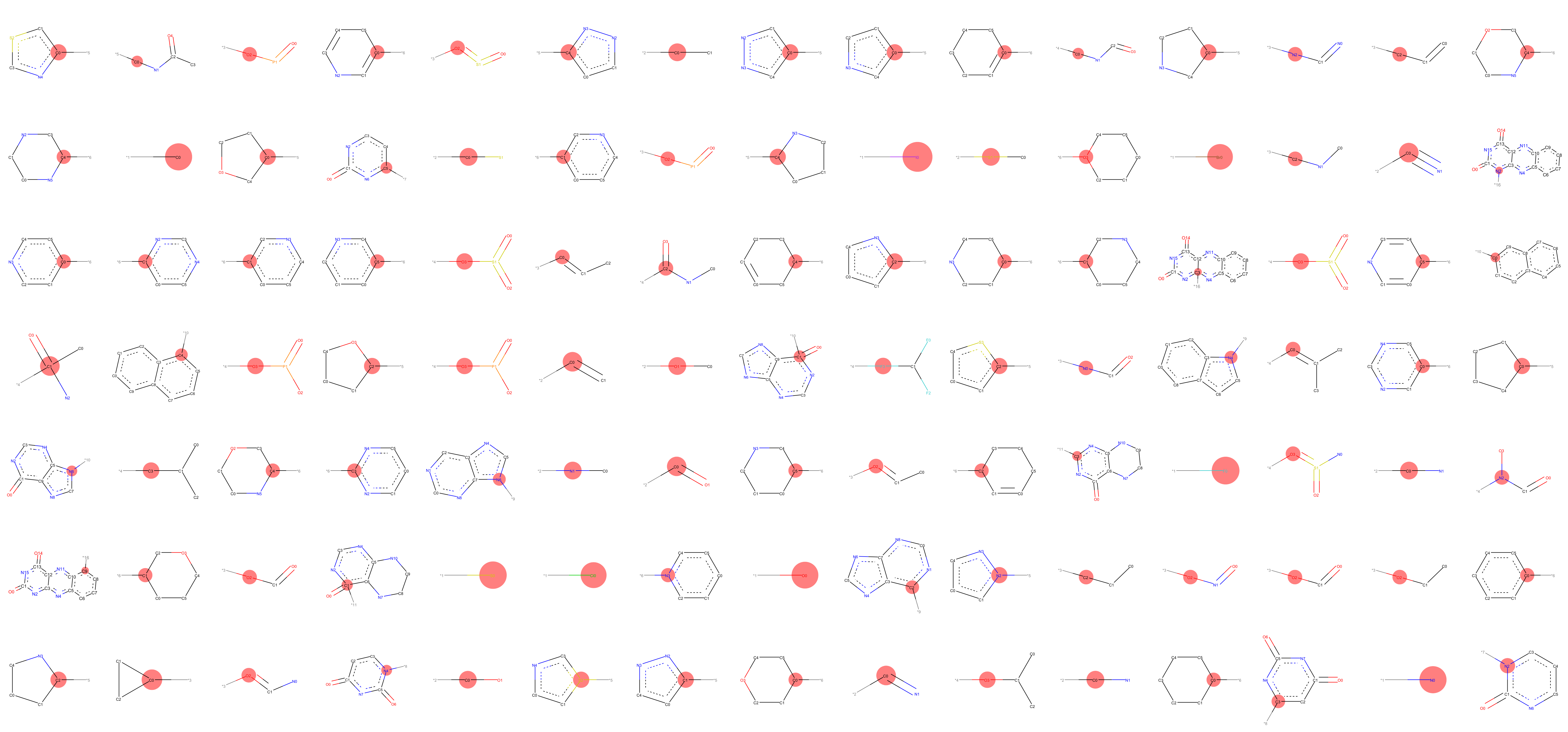}
    \caption{The list of building blocks used in molecule design. The stem, the atom which connects the block to the rest of the molecule, is highlighted.
    }
    \label{fig:blocks}
\end{figure}

We compute the reward based on a proxy's prediction. This proxy is trained on a dataset of 300k randomly generated molecules, whose binding affinity with a target protein has been computed with AutoDock \citep{trott2010autodock}. Since the binding affinity is an energy where lower is better, we takes its opposite and then renormalize it (subtract the mean, divide by the standard deviation) to obtain the reward.

We use the sEH protein and its 4JNC inhibitor. The soluble epoxide hydrolase, or sEH, is a well studied protein which plays a role in respiratory and heart disease, which makes it an interesting pharmacological target and benchmark for ML methods.

Note that we also experimented with other biologically relevant quantities, in particular logP (the n-octanol-water partition coefficient) and QED~\citep{bickerton2012quantifying}. Both were very easy to maximize with GFlowNet. For logP we quickly find molecules with a >20 logP, which at this point is biologically uninteresting (for reference, ibuprofen's logP is between 3.5 and 4). For QED, we also quickly find molecules with the maximum possible QED in our action space, which is 0.948 (in fact our top-1000 is >0.94 after 100k molecules seen). Since docking is a much harder oracle, we focused on it. Also note that we experimented with combining different scores multiplicatively (e.g. multiplying docking score by a renormalized QED and synthesizability), with some success. A more specific contribution in that regards is left to future work.

\subsection{Molecule domain implementation details}
\label{app:impl-details}

For the proxy of the oracle, from which the reward is defined, we use an MPNN~\citep{gilmer2017neural} that receives the atom graph as input. We compute the atom graph using RDKit. Each node in the graph has features including the one-hot vector of its atomic number, its hybridization type, its number of implicit hydrogens, and a binary indicator of it being an acceptor or a donor atom. The MPNN uses a GRU at each iteration as the graph convolution layer is applied iteratively for 12 steps, followed by a Set2Set operation to reduce the graph, followed by a 3-layer MLP. We use 64 hidden units in all of its parts, and LeakyReLU activations everywhere (except inside the GRU).

In the non-active-learning experiments, we train the proxy with a dataset of 300k molecules. To make the task interesting, we use 80\% of molecules obtained from random trajectories, and 20\% obtained from previous runs of RL agents. This extra 20\% contains slightly higher scoring molecules that are varied enough to allow for an interesting challenge. See Fig.~\ref{fig:histo-beta} for the reward distribution. The exact dataset is provided on our github repository for reproducibility. 

For the flow predictor $\flowf$ we also use an MPNN, but it receives the block graph as input. This graph is a tree by construction. Each node in the graph is a learned embedding (each of the 105 blocks has its own embedding and each type of bond has an edge embedding). We again use a GRU over the convolution layer applied 10 times. For each stem of the graph (which represents an atom or block where the agent can attach a new block) we pass its corresponding embedding (the output of the 10 steps of graph convolution + GRU) into a 3-layer MLP to produce 105 logits representing the probability of attaching each block to this stem for MARS and PPO, or representing the flow $\flowf(s,a)$ for GFlowNet; since each block can have multiple stems, this MLP also receives the underlying atom within the block to which the stem corresponds. For the stop action, we perform a global mean pooling followed by a 3-layer MLP that outputs 1 logit for each flow prediction. We use 256 hidden units everywhere as well as LeakyReLU activations.

For further stability we found that multiplying the loss for terminal transitions by a factor $\lambda_T>1$ helped. Intuitively, doing so prioritizes correct predictions at the endpoints of the flow, which can then propagate through the rest of the network/state space via our bootstrapping objective. This is similar to using reward prediction as an auxiliary task in RL (Jaderberg et al., 2017).

Here is a summary of the flow model hyperparameters:
\begin{center}
{\def\arraystretch{1.3}\tabcolsep=2pt
\begin{tabular}{|c|c|c|}
\hline
     Learning rate & $5\times10^{-4}$ & \\
     Minibatch size & 4 & \# of trajectories per SGD step \\
     Adam $\beta,\epsilon$ & (0.9, 0.999), $10^{-8}$ & \\
     \# hidden \& \# embed & 256 &\\
     \# convolution steps & 10 &\\
     Loss $\epsilon$ & $2.5\times 10^{-5}$ & $\epsilon$ in \eqref{eq:flow-loss-log}\\
     Reward $T$ & 8 &\\
     Reward $\beta$ & 10 & $\hat R(x) = (R(x)/T)^\beta$\\
     Random action probability & 0.05 & exploratory factor\\
     $\lambda_T$ & 10 & leaf loss coefficient\\
     $R_{min}$ & 0.01 & {\parbox[t]{4cm}{\small $R$ is clipped below $R_{min}$, i.e. \\ $\hat R_{min}=(R_{min}/T)^\beta$}}\\
\hline
\end{tabular}%
}
    
\end{center}
For MARS we use a learning rate of $2.5\times 10^{-4}$ and for PPO, $1\times 10^{-4}$. For PPO we use an entropy regularization coefficient of $10^{-6}$ and we set the reward $\beta$ to 4 (higher did not help). For MARS we use the same algorithmic hyperparameters as those found in \citet{xie2021mars}. For JT-VAE, we use the code provided by \citet{Jin_2020} as-is, only replacing the reward signal with ours.

\subsection{Multi-Round Experiments}
\label{app:multi-round}
Algorithm~\ref{algo:multi_round} defines the procedure to train the policy $\pi_\theta$ and used in inner loop of the multi-round experiments in the hyper-grid and molecule domains. The effect of diverse generation becomes apparent in the multi-round setting. Since the proxy itself is trained based on the input samples proposed by the generative models (and scored by the oracle, e.g., using docking), if the generative model is not exploratory enough, the reward (defined by the proxy) would only give useful learning signals around the discovered modes. The oracle outcomes $O(x)$ are scaled to be positive, and a hyper-parameter $\beta$ (a kind of inverse temperature) can be used to make the modes of the reward function more or less peaked.

\begin{algorithm}[H]
\label{algo:multi_round}
\SetAlgoLined
\kwInput{Initial dataset $D_0=\{x_i, y_i\}, i=1,\dots,k$\,; $K$ for $TopK$ evaluation; number of rounds (outer loop iterations) $N$;
inverse temperature $\beta$}
\KwResult{A set $TopK(D_N)$ of high valued $x$}
\kwInit{\\
Proxy $M$\; 
Generative policy $\pi_\theta$\; 
Oracle $O$\;
$i=1$\;}

 \While{$i<=N$}{
  Fit $M$ on dataset $D_{i-1}$\;
  Train $\pi_\theta$ with unnormalized probability function $r(x) = M(x)^\beta$ as target reward\;
  Sample query batch $B = \{x_1, \dots, x_b\}$ with $x_i \sim \pi_\theta$\;
  Evaluate batch $B$ with $O$, $\hat{D_i}=\{(x_1, O(x_1)), \dots, (x_b, O(x_b))\}$\;
  Update dataset $D_i=\hat{D_i} \cup D_{i-1}$\;
  $i=i+1$\;
 }
 \caption{Multi-Round Active Learning}
\end{algorithm}

\subsubsection{Hyper-grid}
We use the Gaussian Process implementation from \texttt{botorch}\footnote{\href{http://botorch.org/}{http://botorch.org/}} for the proxy. The query batch size of samples generated after each round is $16$. The hyper-parameters for training the generative models are set to the best performing values from the single-round experiments. 

The initial dataset only contains 4 of the modes. GFlowNet discovered 10 of the modes within 5 rounds, while MCMC discovered 10 within 10 rounds, whereas PPO managed to discover only 8 modes by the end (with $R_0=10^{-1}$). 

\subsubsection{Molecules}
The initial set $D_0$ of $2000$ molecules is sampled randomly from the 300k dataset. At each round, for the MPNN proxy retraining, we use a fixed validation set for determining the stopping criterion. This validation set of $3000$ examples is also sampled randomly from the 300k dataset. We use fewer iterations when fitting the generative model, and the rest of the hyper-parameters are the same as in the single round setting. 

\begin{center}
\label{tab:mr_mol_app}
\begin{tabular}{|c||c|c|}
\hline
     & \multicolumn{2}{c|}{Reward after 1800 docking evaluations} \\
\hline
 method & top-10 & top-100 \\
\hline
GFlowNet & $ 8.83 \pm 0.15 $ & $ 7.76 \pm 0.11 $\\
MARS & $ 8.27 \pm 0.20 $ & $ 7.08 \pm 0.13 $ \\
\hline
\end{tabular}
\end{center}

\begin{figure}[h]
\centering
\begin{tabular}{c|cccc}
\begin{subfigure}{.18\textwidth}
  \centering
  \includegraphics[width=\linewidth]{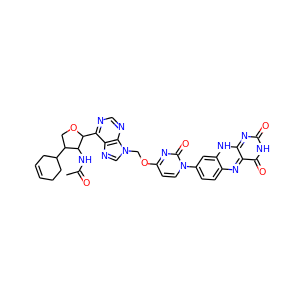}
  \caption{Reward=$8.26$}
  \label{fig:sub1}
\end{subfigure}%
&
\begin{subfigure}{.18\textwidth}
  \centering
  \includegraphics[width=\linewidth]{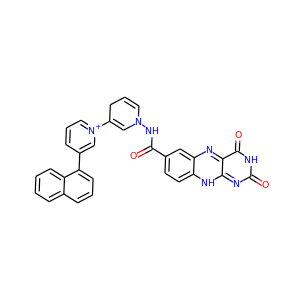}
  \caption{Reward=$9.18$}
  \label{fig:sub12}
\end{subfigure}%
&
\begin{subfigure}{.18\textwidth}
  \centering
  \includegraphics[width=\linewidth]{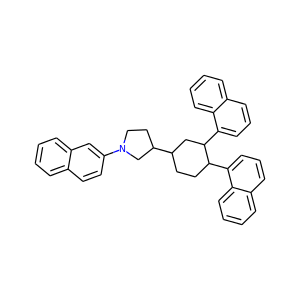}
  \caption{Reward=$9.08$}
  \label{fig:sub13}
\end{subfigure}%
&
\begin{subfigure}{.18\textwidth}
  \centering
  \includegraphics[width=\linewidth]{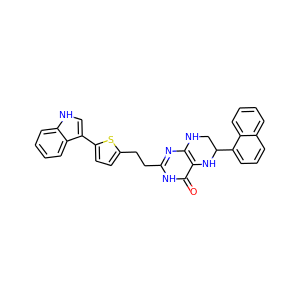}
  \caption{Reward=$8.73$}
  \label{fig:sub14}
\end{subfigure}%
&
\begin{subfigure}{.18\textwidth}
  \centering
  \includegraphics[width=\linewidth]{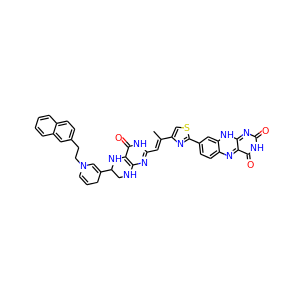}
  \caption{Reward=$8.64$}
  \label{fig:sub15}
\end{subfigure}%
\end{tabular}
\caption{(a) Highest reward molecule in $D_0$ in the multi-round molecule experiments. (b) Highest Reward molecule generated by GFlowNet. (c)-(e) Samples from the top-10 molecules generated by GFlowNet.}
\label{fig:test}
\end{figure}

\subsection{Hypergrid Experiments}
\label{app:toy}

Let's first look at what is learned by GFlowNet. What is the distribution of flows learned? First, in Figure \ref{fig:path_flownet_grid} (Left), we can observe that the distribution learned, $\pi_\theta(x)$, matches almost perfectly $p(x)\propto R(x)$ on a grid where $n=2$, $H=8$. In Figure \ref{fig:path_flownet_grid} (Middle) we plot the visit distribution on all paths that lead to mode $s=(6,6)$, starting at $s_0=(0,0)$. We see that it is fairly spread out, but not uniform: there seems to be some preference towards other corners, presumably due to early bias during learning as well as the position of the other modes. In Figure \ref{fig:path_flownet_grid} (Right) we plot what the uniform distribution on paths from $(0,0)$ to $(6,6)$ would look like for reference. Note that our loss does not enforce any kind of distribution on flows, and a uniform flow is not necessarily desirable (investigating this could be interesting future work, perhaps some distributions of flows have better generalization properties).

\begin{figure}[h]
    \centering
    \includegraphics[width=\linewidth]{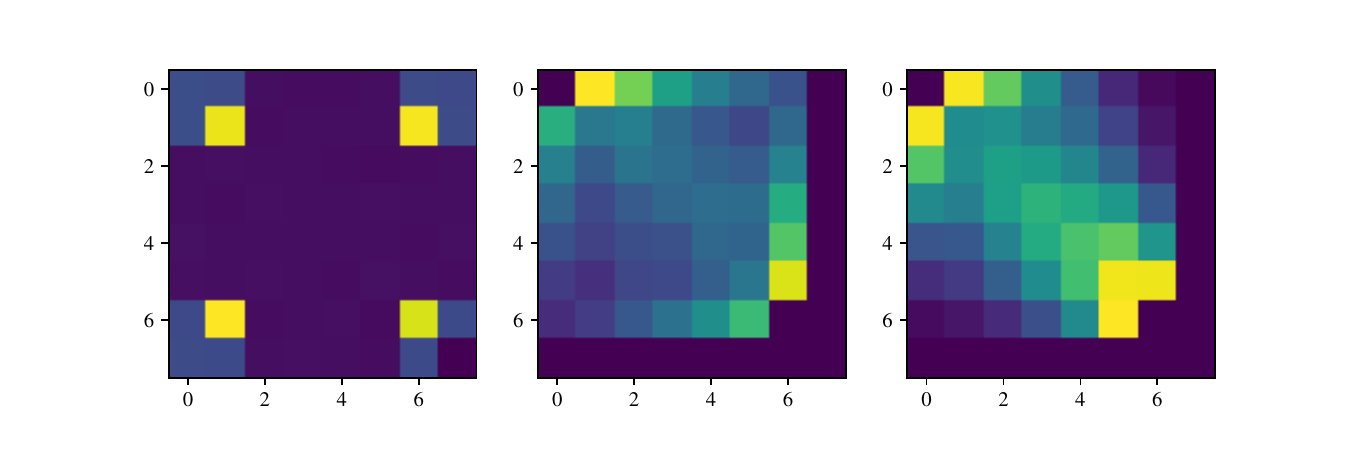}
    \caption{Grid with $n=2$, $H=8$. Left, the distribution $\pi_\theta(x)$ learned on the grid matches $p(x)$ almost perfectly; measured by sampling 30k points. Middle, the visit distribution on sampled paths leading to $(6,6)$. Right, the uniform distribution on all paths leading to $(6,6)$.
    }
    \label{fig:path_flownet_grid}
\end{figure}

Note that we also ran Soft Actor Critic~\citep{haarnoja2018soft} on this domain, but we were unable to find hyperparameters that pushed SAC to find all the modes for $n=4$, $H=8$; SAC would find at best 10 of the 16 modes even when strongly regularized (but not so much so that the policy trivially becomes the uniform policy). While we believe our implementation to be correct, we did not think it would be relevant to include these results in figures, as they are poor but not really surprising: as would be consistent with reward-maximization, SAC quickly finds a mode to latch onto, and concentrates all of its probability mass on that mode, which is the no-diversity failure mode of RL we are trying to avoid with GFlowNet.

Next let's look at the losses as a function of $R_0$, again in the $n=4$, $H=8$ setting. We separate the loss in two components, the leaf loss (loss for terminal transitions) and the inner flow loss (loss for non-terminals). In Figure \ref{fig:grid_losses} we see that as $R_0$ decreases, both inner flow and leaf losses get larger. This is reasonable for two reasons: first, for e.g. with $R_0=10^{-3}$, $\log 10^{-3}$ is a larger magnitude number which is harder for DNNs to accurately output, and second, the terminal states for which $\log 10^{-3}$ is the flow output are $100\times$ rarer than in the $R_0=10^{-1}$ case (because we are sampling states on-policy), thus a DNN is less inclined to correctly predict their value correctly. This incurs rare but large magnitude losses. Note that theses losses are nonetheless, small, in the order of $10^{-3}$ or less, and at this point the distribution is largely fit and the model is simply converging.

\begin{figure}[h]
    \centering
    \includegraphics[width=\linewidth]{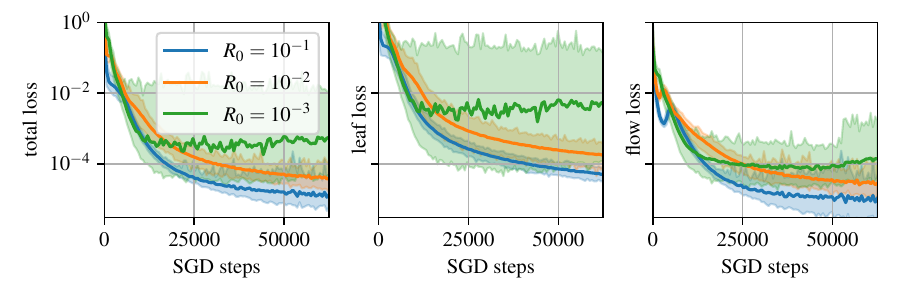}
    \caption{Losses during training for the ``corners'' reward function in the hypergrid, with $n=4$, $H=8$. Shaded regions are the min-max bounds.
    }
    \label{fig:grid_losses}
\end{figure}

\paragraph{GFlowNet as an offline off-policy method}
To demonstrate this feature of GFlowNet, we train it on a fixed dataset of trajectories and observe what the learned distribution is.  For this experiment we use $R(x)=0.01 + \prod_i (\cos(50x_i) + 1) f_\mathcal{N}(5x_i)$, $f_\mathcal{N}$ is the normal p.d.f., $n=2$ and $H=30$. We show results for two random datasets. First, in Figure \ref{fig:flow_dataset_rtraj} we show what is learned when the dataset is sampled from a uniform random policy, and second in Figure \ref{fig:flow_dataset_utraj} when the dataset is created by sampling points uniformly on the grid and walking backwards to the root to generate trajectories. The first setting should be much harder than the second, and indeed the learned distribution matches $p(x)$ much better when the dataset points are more uniform.
Note that in both cases many points are left out intentionally as a generalization test.

\begin{figure}[h]
    \centering
    \includegraphics[width=\linewidth]{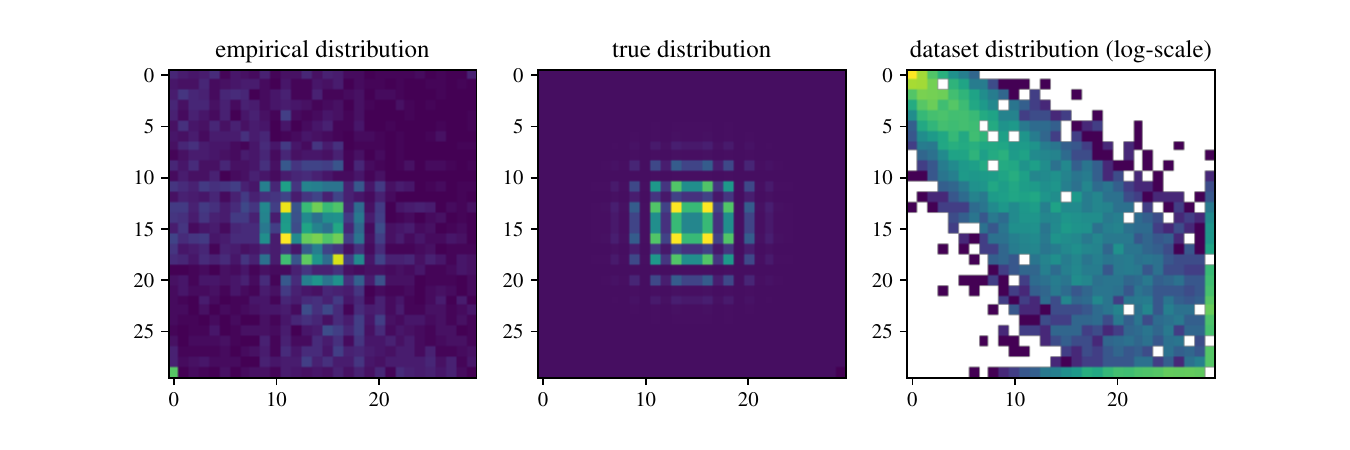}
    \caption{Grid with $n=2$, $H=30$. Left, the learned distribution $\pi_\theta(x)$. Middle, the true distribution. Right, the dataset distribution, here generated by executing a uniform random policy from $s_0$.
    }
    \label{fig:flow_dataset_rtraj}
\end{figure}
\begin{figure}[h]
    \centering
    \includegraphics[width=\linewidth]{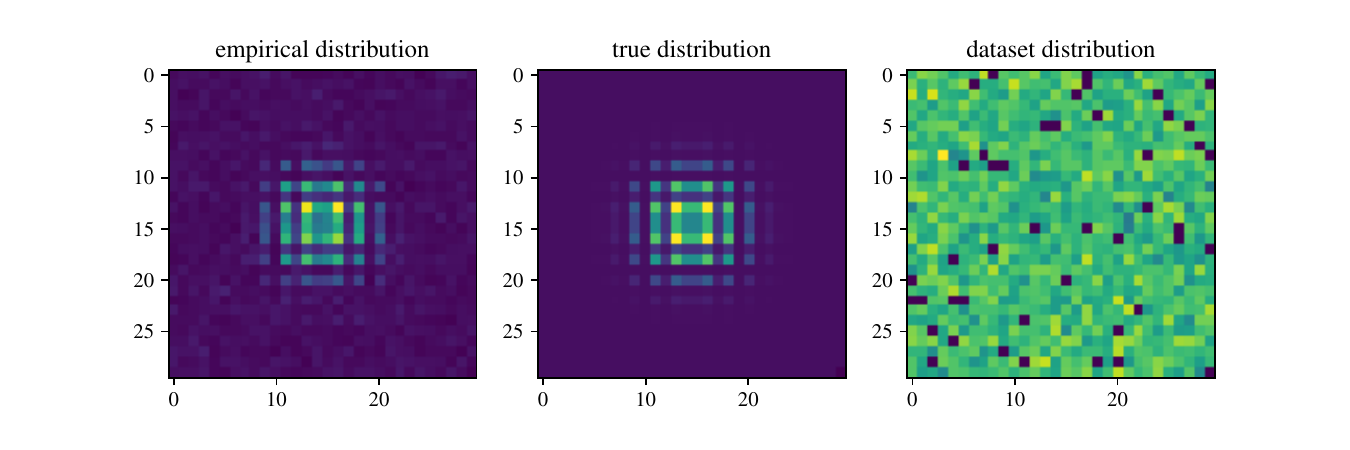}
    \caption{Grid with $n=2$, $H=30$. Left, the learned distribution $\pi_\theta(x)$. Middle, the true distribution. Right, the dataset distribution, here generated by sampling a point uniformly on the grid and sampling random parents until $s_0$ is reached, thus generating a training trajectory in reverse.
    }
    \label{fig:flow_dataset_utraj}
\end{figure}

These results suggest that GFlowNet can easily be applied offline and off-policy. Note that we did not do hyperparameter search on these two plots, these are purely illustrative and we believe it is likely that better generalization can be achieved by tweaking hyperparameters.

\subsection{GFlowNet results on the molecule domain}
\label{app:flownet-more-results}

Here we present additional results to give insights on what is learned by our method, GFlowNet.

Let's first examine the numerical results of Figure \ref{fig:top-k}:
\begin{center}
{\renewcommand{\arraystretch}{1.2}%
\begin{tabular}{|c||c|c|c|}
\hline
     & \multicolumn{3}{c|}{Reward at $10^5$ samples} \\
\hline
 method & top-10 & top-100 & top-1000 \\
\hline
GFlowNet & $ 8.36 \pm 0.01 $ & $ 8.21 \pm 0.03 $ & $ 7.98 \pm 0.04 $\\
MARS & $ 8.05 \pm 0.12 $ & $ 7.71 \pm 0.09 $ & $ 7.13 \pm 0.19 $\\
PPO & $ 8.06 \pm 0.26 $ & $ 7.87 \pm 0.29 $ & $ 7.52 \pm 0.26 $\\
\hline
& \multicolumn{3}{c|}{Reward at $10^6$ samples}\\
\hline
GFlowNet &  $ 8.45 \pm 0.03 $ & $ 8.34 \pm 0.02 $ & $ 8.17 \pm 0.02 $\\
MARS & $ 8.31 \pm 0.03 $ & $ 8.03 \pm 0.08 $ & $ 7.64 \pm 0.16 $\\
PPO &  $ 8.25 \pm 0.12 $ & $ 8.08 \pm 0.12 $ & $ 7.82 \pm 0.16 $\\
\hline
& \multicolumn{3}{c|}{Reward for $10^6$-equivalent compute}\\
\hline
JT-VAE + BO & $6.03$ & $ 5.86 $ & $5.31 $\\
\hline
\end{tabular}}
\end{center}
These are means and standard deviations computed over 3 runs. We see that GFlowNet produces significantly better molecules. It also produces much more diverse ones: GFlowNet has a mean pairwise Tanimoto similarity for its top-1000 molecules of $0.44\pm 0.01$, PPO, $0.62\pm0.03$, and MARS, $0.59\pm 0.02$ (mean and std over 3 runs). A random agent for this environment would yield an average pairwise similarity of $0.231$ (and very poor rewards). 

\begin{figure}[h]
    \centering
    \includegraphics[width=\linewidth]{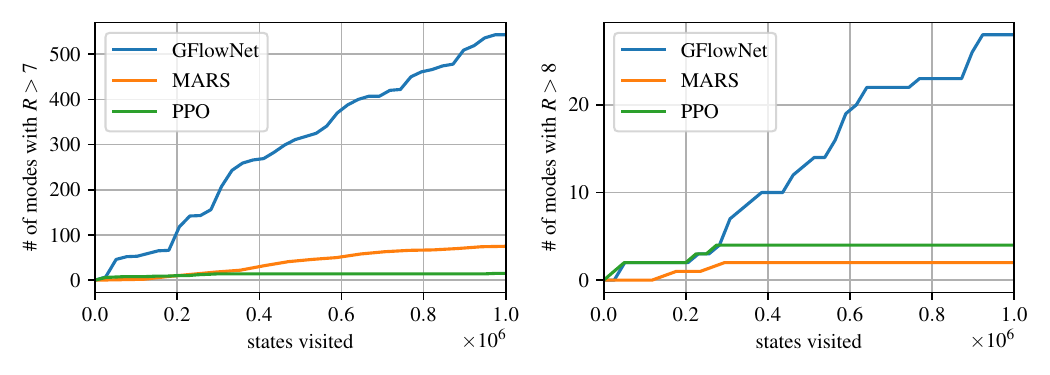}
    \caption{Number of Tanimoto-separated modes found above reward threshold $T$ as a function of the number of molecules seen. See main text. Left, $T=7$. Right, $T=8$. 
    }
    \label{fig:mol-modes}
\end{figure}

\begin{figure}[h]
    \centering
    \includegraphics[width=\linewidth]{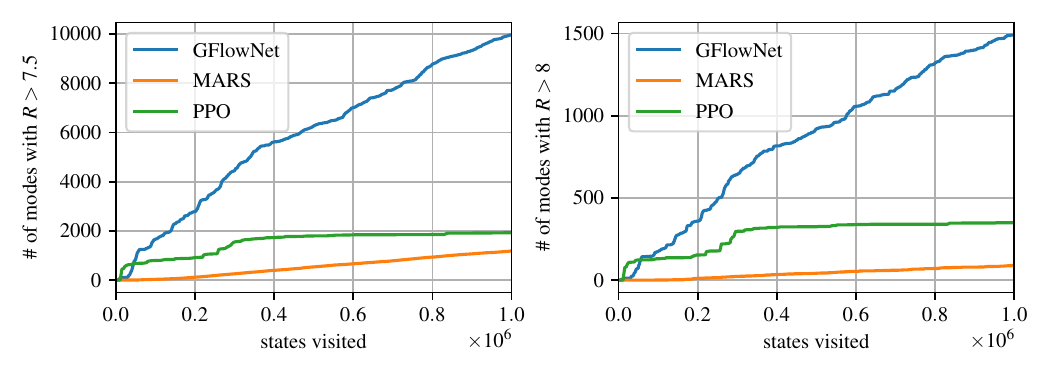}
    \caption{Number of diverse Bemis-Murcko scaffolds \citep{bemis1996properties} found above reward threshold $T$ as a function of the number of molecules seen. Left, $T=7.5$. Right, $T=8$. 
    }
    \label{fig:mol-modes-scaff-app}
\end{figure}

We also see that GFlowNet produces much more diverse molecules by approximately counting the number of modes found within the high-reward molecules. Here, we define "modes" as molecules with an energy above some threshold $T$, at most similar to each other in Tanimoto space at threshold $S$. In other words, we consider having found a new mode representative when a new molecule has a Tanimoto similarity smaller than $S$ to every previously found mode's representative molecule. We choose a Tanimoto similarity $S$ of 0.7 as a threshold, as it is commonly used in medicinal chemistry to find similar molecules, and a reward threshold of 7 or 8. We plot the results in Figure \ref{fig:mol-modes}. We see that for $R>7$, GFlowNet discovers many more modes than MARS or PPO, over 500, whereas MARS only discovers less than 100.

Another way to approximate the number of modes is to count the number of diverse Bemis-Murcko scaffolds\bibentry{bemis1996properties} present within molecules above a certain reward threshold. We plot these counts in Figure \ref{fig:mol-modes-scaff}, where we again see that GFlowNet finds a greater number of modes.

Next, let's try to understand what is learned by GFlowNet. In a large scale domain without access to $p(x)$, it is non-trivial to demonstrate that $\pi_\theta(x)$ matches the desired distribution $p(x)\propto R(x)$. This is due to the many-paths problem: to compute the true $p_\theta(x)$ we would need to sum the $p_\theta(\tau)$ of all the trajectories that lead to $x$, of which there can be an extremely large number. Instead, we show various measures that suggest that the learned distribution is consistent with the hypothesis the $\pi_\theta(x)$ matches $p(x)\propto R(x)^\beta$ well enough.

In Figure \ref{fig:leaf-scatter} we show how $\flowf_\theta$ partially learns to match $R(x)$. In particular we plot the inflow of leaves (i.e. for leaves $s'$ the $\sum_{s,a:T(s,a)=s'}\flowf(s,a)$) as versus the target score ($R(x)^\beta$). 

\begin{figure}[h]
    \centering
    \includegraphics[width=0.8\linewidth]{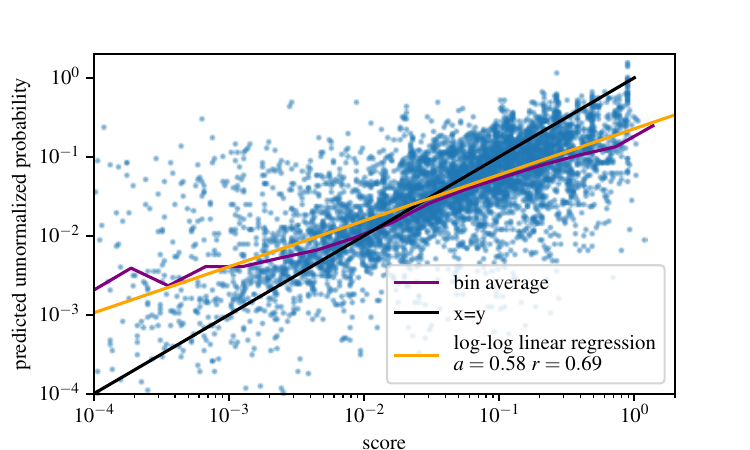}
    \caption{Scatter of the score ($R(x)^\beta$) vs the inflow of leaves (the predicted unnormalized probability). The two should match. We see that a log-log linear regression has a slope of $0.58$ and a $r$ of $0.69$. The slope being less than 1 suggests that GFlowNet tends to underestimate high rewards (this is plausible since high rewards are visited less often due to their rarity), but nonetheless reasonably fits its data. Here $\beta=10$. We plot here the last 5k molecules generated by a run.
    }
    \label{fig:leaf-scatter}
\end{figure}

Another way to view that the learned probabilities are self-consistent is that the histograms of $R(x)/Z$ and $\hat p_\theta(x)/Z$ match, where we use the predicted $Z= \sum_{a\in \mathcal{A}(s_0)} \flowf(s_0,a)$, and $\hat p_\theta(x)$ is the inflow of the leaf $x$ as above. We show this in Figure \ref{fig:leaf-histo}.
\begin{figure}[h]
    \centering
    \includegraphics[width=0.8\linewidth]{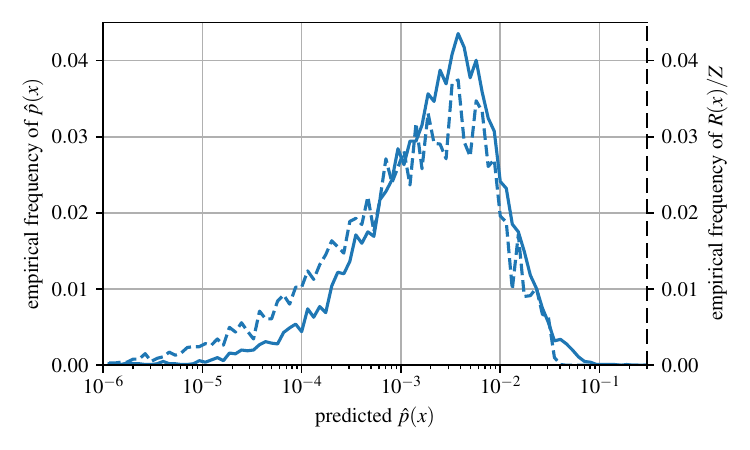}
    \caption{Histogram of the predicted density vs histogram of reward. The two should match. We compute these with the last 10k molecules generated by a run. This plot again suggests that the model is underfitted. It thinks the low-reward molecules are less likely than they actually are, or vice-versa that the low-reward molecules are better than they actually are. This is consistent with the previous plot showing a lower-than-1 slope.
    }
    \label{fig:leaf-histo}
\end{figure}

In terms of loss, it is interesting that our models behaves similarly to value prediction in deep RL, in the sense that the value loss never goes to 0. This is somewhat expected due to bootstrapping, and the size of the state space. Indeed, in our hypergrid experiments the loss does go to 0 as the model converges. We plot the loss separately for leaf transitions (where the inflow is trained to match the reward) and inner flow transitions (at visited states, where the inflow is trained to match the outflow) in Figure \ref{fig:losses}.

\begin{figure}[h]
    \centering
    \includegraphics[width=0.8\linewidth]{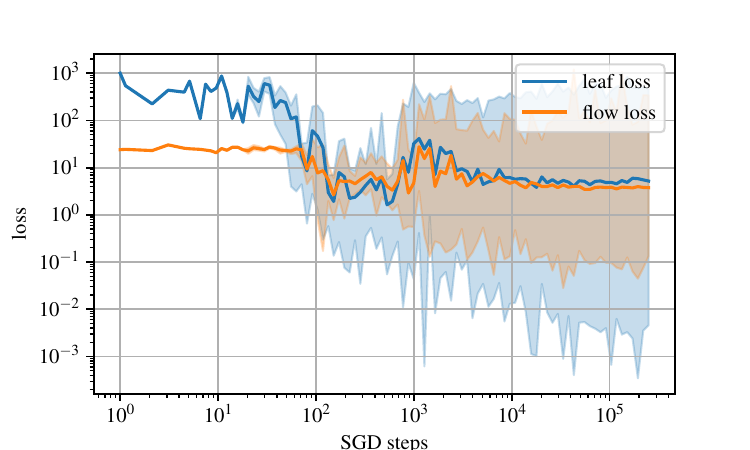}
    \caption{Loss as a function of training for a typical run of GFlowNet on the molecule domain. The shaded regions represent the min-max over the interval. We note several phases: In the initial phase the scale of the predictions are off and the leaf loss is very high. As prediction scales adjust we observe the second phase where the flow becomes consistent and we observe a dip in the loss. Then, as the model starts discovering more interesting samples, the loss goes up, and then down as it starts to correctly fit the flow over a large variety of samples. The lack of convergence is expected due to the massive state space; this is akin to value-based methods in deep RL on domains such as Atari.
    }
    \label{fig:losses}
\end{figure}

\end{document}